\documentclass{article}

\usepackage{iclr2026_conference,times}
\iclrfinalcopy

\usepackage[utf8]{inputenc} 
\usepackage[T1]{fontenc}    
\usepackage{hyperref}       
\usepackage{url}            
\usepackage{booktabs}       
\usepackage{amsfonts}       
\usepackage{nicefrac}       
\usepackage{microtype}      
\usepackage{xcolor}         
\usepackage{multirow}

\usepackage{amsmath}
\usepackage{amssymb}
\usepackage{mathtools}
\usepackage{amsthm}

\usepackage{mdframed}

\usepackage[capitalize,noabbrev]{cleveref}

\usepackage{cancel}
\usepackage{algorithm}
\usepackage{algorithmic}
\usepackage{stmaryrd}
\usepackage{booktabs}
\usepackage{array}
\usepackage{float}
\usepackage{mathtools}
\usepackage{tikz}
\usepackage{adjustbox}
\usetikzlibrary{positioning, shapes.geometric, arrows.meta}

\tikzstyle{block} = [rectangle, draw, fill=blue!20, text width=8em, text centered, rounded corners, minimum height=3em]
\tikzstyle{cnn} = [rectangle, draw, fill=orange!30, text width=10em, text centered, rounded corners, minimum height=3em]
\tikzstyle{transformer} = [rectangle, draw, fill=green!20, text width=10em, text centered, rounded corners, minimum height=3em]
\tikzstyle{arrow} = [thick, ->, >=stealth]

\usepackage{stfloats}
\usepackage{caption}
\usepackage{subcaption}
\usepackage{verbatim}

\usepackage{shortcut}
\usepackage{wrapfig}

\usepackage{xspace}
\usepackage{pifont} 
\usepackage{makecell}
\usepackage{nicefrac}
\usepackage{xcolor}
\usepackage{placeins} 

\definecolor{darkcerulean}{rgb}{0.03, 0.27, 0.49} 
\definecolor{iris}{rgb}{0.35, 0.31, 0.81} 
\definecolor{linkcolor}{RGB}{82,82,192} 
\definecolor{good_color}{rgb}{0.13850039, 0.41331206, 0.74052025}
\definecolor{bad_color}{rgb}{0.66080672, 0.21526712, 0.23069468}
\definecolor{green_color}{RGB}{100, 166, 113}

\newcommand{\filtersize}{f}

\newcommand{\method}{{PSDNorm}\xspace}
\title{\method: Temporal Normalization for Deep Learning in Sleep Staging}

\theoremstyle{plain}
\newtheorem{theorem}{Theorem}[section]
\newtheorem{proposition}[theorem]{Proposition}

\theoremstyle{definition}

\theoremstyle{remark}

\author{%
  Théo Gnassounou \\
  Université Paris-Saclay,\\
  Inria, CEA,\\
  91120 Palaiseau, France \\
  \texttt{theo.gnassounou@inria.fr} \\
  \And
  Antoine Collas \\
  Université Paris-Saclay,\\
  Inria, CEA,\\
  91120 Palaiseau, France \\
  \texttt{antoine.collas@inria.fr} \\
  \AND
  Rémi Flamary \\
  École Polytechnique, \\
  IP Paris, CMAP, UMR 7641, \\
  91120 Palaiseau, France\\
  \texttt{remi.flamary@polytechnique.edu} \\
  \And
  Alexandre Gramfort\thanks{A. Gramfort joined Meta and can be reached at \texttt{agramfort@meta.com}} \\
  Université Paris-Saclay,\\
  Inria, CEA,\\
  91120 Palaiseau, France \\
  \texttt{alexandre.gramfort@inria.fr} \\
}

\begin{document}

\maketitle

\begin{abstract}
    Distribution shift poses a significant challenge in machine learning, particularly in
    biomedical applications using data collected across different subjects, institutions, and recording devices, such as sleep data.
    While existing normalization layers, BatchNorm, LayerNorm and InstanceNorm, help mitigate distribution shifts, when applied over the time dimension they ignore the dependencies and auto-correlation inherent to the vector coefficients they normalize.
    In this paper, we propose PSDNorm that leverages Monge mapping and temporal context to normalize feature maps in deep learning models for signals.
    Evaluations with architectures based on U-Net or transformer backbones trained on 10K subjects across 10 datasets,
    show that PSDNorm achieves state-of-the-art performance on unseen left-out datasets while being {more robust to data scarcity}.
\end{abstract}

\section{Introduction}
\label{sec:intro}

\paragraph{Data Shift in Physiological Signals}
{Machine learning techniques have achieved remarkable success in various domains, 
including computer vision, biology, audio processing, and language understanding. 
However, these methods face significant challenges when there are distribution shifts 
between training and evaluation datasets~\citep{MORENOTORRES2012521}. For example, in biological data, such as 
electroencephalography (EEG) signals, the distribution of the data can vary significantly.
Indeed, data is collected from different subjects, electrode positions, and recording conditions.
This paper focuses on sleep staging, a clinical task that consists in classifying periods of sleep in different stages based on EEG signals~\citep{STEVENS200445}. 
Depending on the dataset, the cohort can be composed of different age groups, sex repartition, health conditions, and recording conditions~\citep{MASS,SHHS,marcus2013CHAT}.
Such variability brings shift in the distribution making it challenging for the model to generalize to unseen datasets.
}
\paragraph{Normalization to Address Data Shift}

{ Normalization layers are widely used in deep learning to improve training
stability and generalization. Common layers include
BatchNorm~\citep{ioffe2015batch}, LayerNorm~\citep{ba2016layer}, and
InstanceNorm~\citep{ulyanov2016instance}, which respectively compute statistics
across the batch, normalize across all features within each sample, and
normalize each channel independently within a sample. Some normalization methods
target specific tasks, such as EEG covariance matrices~\citep{kobler-etal:22} or
time-series forecasting~\citep{kim2021reversible}, but they do not fully address
spectral distribution shifts reflected in the temporal auto-correlations of
signals.
{Other papers have proposed to adapt layer statistics to new domains~\citep{li2016revisitingbatchnormalizationpractical, chang2019domainspecificbatchnormalizationunsupervised}.}
In sleep staging, a simple normalization is often applied as
preprocessing, e.g., standardizing signals over entire
nights~\citep{apicella2023effects} or short temporal
windows~\citep{chambon2018deep}. Recent
studies~\citep{gnassounou2023convolutional, gnassounou2024multi} highlight the
importance of considering temporal correlation and spectral content in
normalization, proposing Temporal Monge Alignment (TMA), which aligns Power
Spectral Density (PSD) to a common reference using Monge mapping, going beyond
simple z-score normalization. However, these methods remain preprocessing steps
that cannot be inserted as layers in the network architecture as it is done with
BatchNorm, LayerNorm or InstanceNorm. }

\paragraph{Deep Learning for Sleep Staging}
{
Sleep staging has been addressed by various neural network architectures, which process raw signals~\citep{chambon2018deep, Perslev2021USleep, guillot2021robustsleepnet}, spectrograms~\citep{phan_l-seqsleepnet_2023, phan_seqsleepnet_2019}, or both~\citep{xsleepnet}. More recent approaches involve transformer-based models that handle multimodal~\citep{wang_caresleepnet_2024}, spectrogram~\citep{phan2022sleeptransformer}, or heterogeneous inputs~\citep{Guo2024transformer}, offering improved modeling of temporal dependencies. However, most existing models are trained on relatively small cohorts, typically consisting of only a few hundred subjects, which limits their ability to generalize to diverse clinical settings. Notable exceptions include U-Sleep~\citep{Perslev2021USleep}, which was trained on a large-scale dataset and incorporates BatchNorm layers to mitigate data variability, and foundational models~\citep{thapa_multimodal_2025, fox_foundational_2025, deng_unified_2025} that achieve strong generalization from vast amount of data but require significant computational resources and are challenging to adapt without fine-tuning. Our focus is on developing smaller, efficient models that balance good generalization with ease of training and deployment in clinical practice.
}

\begin{figure*}
    \centering
    \includegraphics[width=0.95\textwidth]{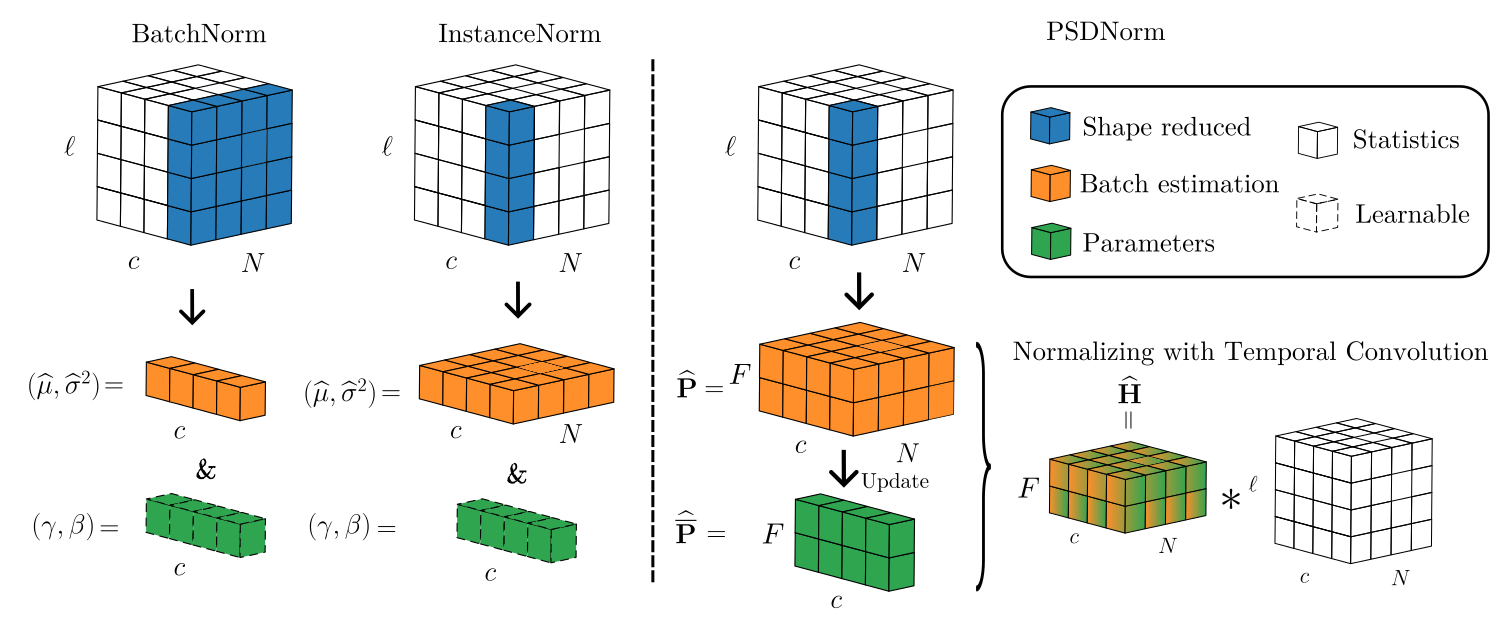}
    \caption{
        \textbf{Description of normalization layers.} 
        The input shape is $(N, c, \ell)$ with batch size $N$, channels $c$, and signal length $\ell$. 
        BatchNorm estimates the mean $\widehat\mu$ and variance $\widehat{\sigma}^2$ over batch and time, 
        and learns parameters $(\gamma, \beta)$ to normalize the input. 
        \method estimates PSDs $\widehat{\mathbf{P}}$ over time and accounts for local temporal correlations. 
        It computes the barycenter PSD $\widehat{\overline{\mathbf{P}}}$, updates it via a running 
        Riemannian barycenter~\eqref{eq:running_barycenter_method}, and applies the filter 
        $\widehat{\mathbf{H}}$ to normalize the input.
        The hyperparameter $\filtersize$ controls the extent of temporal correlation considered, thereby adjusting the strength of the normalization.
        \vspace{-0.3cm}
    }
    \label{fig:concept}
\end{figure*}

\paragraph{Contributions}
In this work, we introduce the \method deep learning layer, 
a novel approach to address distribution shifts in machine learning for signals.
\method leverages Monge Mapping to incorporate temporal context and normalize feature maps effectively.
{This layer enhances model robustness to new subjects at inference time.}
Unlike standard normalization layers such as LayerNorm or InstanceNorm, \method leverages the sequential nature of intermediate feature maps, as illustrated
in Figure \ref{fig:concept}.
We evaluate \method through extensive experiments on 10 
sleep datasets.
{This evaluation covers 10M of samples across 10K subjects, 
using a leave-one-dataset-out (LODO) protocol with 3 different random seeds. 
To the best of our knowledge, such a large-scale and systematic evaluation 
has never been conducted before.}
\method achieves state-of-the-art performance and requires 4 times fewer labeled data to match the accuracy of the best baseline.
Results highlight the potential of \method as a practical and efficient solution for tackling domain shifts in signals. \\
The paper is structured as follows: \Cref{sec:related} discusses existing normalization layers and pre-processing.
\Cref{sec:method} introduces \method, followed by numerical results in \Cref{sec:numerical_results}.

\paragraph*{Notations}
\noindent Vectors are denoted by small cap boldface letters (e.g., $\bx$), matrices by large cap boldface letters (e.g., $\bX$).
The element-wise product, power of $n$ and division are denoted $\odot$, $\cdot^{\odot n}$ and $\oslash$, respectively.
$\intset{K}$ denotes $\{1, \ldots, K\}$. 
The absolute value is $|.|$.
The discrete circular convolution along the temporal axis operates row-wise as, \( *: \mathbb{R}^{c \times \ell} \times \mathbb{R}^{c \times \filtersize} \to \mathbb{R}^{c \times \ell} \) for \( \ell \geq \filtersize \).
$\Vectr: \bbR^{c \times \ell} \to \bbR^{c \ell}$ concatenates rows of a time series into a vector.
$x_l = [\bx]_{l}$ refers to the $l^\text{th}$ element of $\bx$, and $X_{l,m} = [\bX]_{l,m}$ denotes the element of $\bX$ at the $l^\text{th}$ row and $m^\text{th}$ column. 
$\bX^*$ and $\bX^\top$ are the conjugate and the transpose of $\bX$, respectively.
$\diag$ puts the elements of a vector on the diagonal of a matrix.
$\otimes$ is the Kronecker product.
$\bone_c$ is the vector of ones of size $c$.

\section{Related Works}

In this section, we first review {classical architectures for sleep staging} and fundamental concepts of normalization layers.
Then, we recall the Temporal Monge Alignment (TMA) method~\citep{gnassounou2023convolutional} that aligns the PSD of signals using optimal transport.

\label{sec:related}

\paragraph{Deep Learning for Sleep Staging}
Numerous neural network architectures have been proposed for sleep staging, processing data in different formats as introduced in \cref{sec:intro}.
Different types of architectures have been explored, such as convolutional neural networks (CNNs)~\citep{chambon2018deep}, recurrent neural networks (RNNs)~\citep{Supratak_2017, phan_seqsleepnet_2019}, and more recently transformers~\citep{phan2022sleeptransformer, wang_caresleepnet_2024, Guo2024transformer}, which have shown promise in modeling temporal dependencies in sleep data.
While many models are typically evaluated on a limited number of datasets, the work by~\citep{Perslev2021USleep} introduced U-Sleep, a model trained on a large-scale dataset of sleep recordings.
Their architecture, based on U-Time~\citep{perslev2019u}, incorporates BatchNorm layers to mitigate data variability, and they employ a domain generalization approach: training a single model on a sufficiently diverse set of domains to ensure it generalizes to unseen datasets without additional adaptation.
This architecture is composed of encoder-decoder blocks with skip connections, allowing the model to capture both local and global features of the sleep signals effectively.
Each encoder and decoder block consists of convolutional layers followed by BatchNorm and non-linear activation functions, enabling the model to learn robust representations of the input data.
A more detailed description of U-Time is provided in \cref{app:U-Sleep}.
\color{black}
\paragraph{Normalization Layers}

{Normalization layers enhance training and robustness in deep neural networks. 
The most common are BatchNorm~\citep{ioffe2015batch}, InstanceNorm~\citep{ulyanov2016instance}, and LayerNorm~\citep{ba2016layer}.
BatchNorm normalizes feature maps using batch and time statistics, 
ensuring zero mean and unit variance. The output is adjusted with learnable parameters.
InstanceNorm normalizes each channel per sample using its own statistics, 
independent of the batch (see Fig.~\ref{fig:concept}). Popular in time-series forecasting, it is used in RevIN~\citep{kim2021reversible}, 
which reverses normalization after decoding.
LayerNorm normalizes across all channels and time steps within each sample, 
with learnable scaling and shifting.}
{While these normalization layers are widely employed, they operate on vectors ignoring statistical dependence and autocorrelation between their coefficients, which are prevalent when operating on time-series.}
To address this limitation, the Temporal Monge Alignment (TMA)~\citep{gnassounou2023convolutional, gnassounou2024multi} was introduced as
a pre-processing step to align temporal correlations by leveraging the Power Spectral density (PSD) of
multivariate signals using Monge Optimal Transport mapping.

\paragraph{Gaussian Periodic Signals}
Consider a multivariate signal $\bX \triangleq [\bx_1, \dots, \bx_c]^\top \in \bbR^{c \times \ell}$ of sufficient length.
A standard assumption is that this signal follows a centered Gaussian
distribution where sensors are uncorrelated and signals are periodic.
This periodicity and uncorrelation structure implies that the signal's covariance matrix is block diagonal, with each block having a circulant structure.
A fundamental property of symmetric positive definite circulant matrices is their diagonalization~\citep{gray06} with real and positive eigenvalues in the Fourier basis $\bF_{\ell} \in \bbC^{\ell \times \ell}$ of elements
\begin{equation}
    \label{eq:fourier_basis}
    \left[\bF_{\ell}\right]_{l,l^\prime} \triangleq \frac{1}{\sqrt{\ell}}\exp\left(-2i \pi \frac{(l-1) (l^\prime-1)}{\ell}\right),
\end{equation}
where $l,l^\prime \in \intset{\ell}$.
Hence, we have $\Vectr(\bX)\sim\mathcal{N}(\mathbf{0},\bSigma)$ with $\bSigma$ block-diagonal,
\begin{equation}
    \label{eq:covariance_structure}
    \bSigma = \left(\bI_{c} \otimes \bF_{\ell}\right) \diag\left(\mathrm{vec}(\bP)\right) \left(\bI_{c} \otimes \bF_{\ell}^{*}\right)
    \;\in\;\mathbb{R}^{c\ell\times c\ell},
\end{equation}
where $\bP \in \mathbb{R}^{c\times \ell}$ contains positive entries
corresponding to the Power Spectral Density of each sensor.
In practice, since we only have access to a single realization of the signal, the PSD is estimated with only $\filtersize \ll \ell$ frequencies, \ie $\bP \in \bbR^{c \times \filtersize}$.
This amounts to considering the local correlation of the signal and neglecting the long-range correlations.

\paragraph{Power Spectral Density Estimation}
The Welch estimator~\citep{welch1967use} computes the PSD of a signal by
averaging the squared Fourier transform of overlapping segments of the signal.
Hence, the realization of the signal $\bX \in \bbR^{c \times \ell}$ is decimated into overlapping segments $\{\bX^{(1)}, \dots, \bX^{(L)}\}\subset \bbR^{c \times \filtersize}$ to estimate the PSD.
The Welch estimator is defined as
\begin{equation}
    \widehat{\bP} \triangleq \frac{1}{L} \sum_{l=1}^{L} \left|\left(\left(\bone_c \bw^\top\right) \odot \bX^{(l)}\right) \bF_{\filtersize}^* \right|^{\odot 2} \in \bbR^{c \times \filtersize} \;,
    \label{eq:welch}
\end{equation}
where $\bw \in \bbR^\filtersize$ is the window function such that $\Vert \bw \Vert_2 = 1$.

\paragraph{$\filtersize$-Monge Mapping}
Let $\cN(\mathbf{0}, \bSigma^{(s)})$ and $\cN(\mathbf{0}, \bSigma^{(t)})$ 
be source and target centered Gaussian distributions respectively with covariance matrices following the structure~\eqref{eq:covariance_structure} and PSDs denoted by $\bP^{(s)}$ and $\bP^{(t)} \in \bbR^{c \times \filtersize}$.
Given a signal $\bX \in \bbR^{c\times \ell}$ such that $\Vectr(\bX) \sim
\cN(\mathbf{0}, \bSigma^{(s)})$, the $\filtersize$-Monge mapping as defined by \citep{gnassounou2023convolutional, gnassounou2024multi} is
\begin{equation}
    \label{eq:monge_mapping}
    m_{\filtersize}\left(\bX, \bP^{(t)}\right) \triangleq \bX * \bH \in \bbR^{c \times \ell},\quad \text{where}\quad   \bH \triangleq \frac{1}{\sqrt{\filtersize}} 
    \left( {\bP^{(t)}} \oslash {\bP^{(s)}} \right)^{\odot\frac12} \bF_{\filtersize}^*
    \in \bbR^{c\times \filtersize} \;.
\end{equation}
In this case, $\filtersize$ controls the alignment between the source and target distributions.
Indeed, if $\filtersize = \ell$, then the $\filtersize$-Monge mapping is the classical Monge mapping between Gaussian distributions and the source signal has its covariance matrix equal to $\bSigma^{(t)}$ after the mapping.
If $\filtersize = 1$, then each sensor is only multiplied by a scalar.

\paragraph{Gaussian Wasserstein Barycenter}
For Gaussian distributions admitting the decomposition~\eqref{eq:covariance_structure}, the Wasserstein barycenter~\citep{agueh2011barycenters} admits an elegant closed-form solution.
Consider $K$ centered Gaussian distributions admitting the decomposition~\eqref{eq:covariance_structure} of PSDs $\bP^{(1)}, \dots, \bP^{(K)}$.
Their barycenter is also a centered Gaussian distribution $\cN(\mathbf{0}, \overline{\bSigma})$ admitting the decomposition~\eqref{eq:covariance_structure} with PSD
\begin{equation}
        \overline{\bP} 
        \triangleq \left( \frac{1}{K} \sum^{K}_{k=1} {\bP^{(k)}}^{\odot\frac{1}{2}}\right)^{\odot 2} \in \bbR^{c \times \filtersize} \;.
        \label{eq:barycenter}
\end{equation}

\paragraph{Temporal Monge Alignement}
TMA is a pre-processing method that aligns the PSD of multivariate signals using the $\filtersize$-Monge mapping.
Given a source signal $\bX_s$ and a set of target signals $\bX_t = \{\bX_{t}^{(1)}, \dots, \bX_{t}^{(K)}\}$, the TMA method uses the $\filtersize$-Monge mapping between the source and the Wasserstein barycenter of the target signals.
Hence, it simply consists of 1) estimating the PSD of all the signals, 2)
computing the Wasserstein barycenter of the target signals, and 3) applying the $\filtersize$-Monge mapping to the source signal. 
TMA, as a preprocessing method, is inherently limited to handling PSD shifts in
the raw signals and cannot address more complex distributional changes in the data.
This limitation highlights the need for a layer that can effectively capture and adapt to these complex variations during learning and inside deep learning models.

\section{\method Layer}
\label{sec:method}

The classical normalization layers, such as BatchNorm or InstanceNorm do not
take into account the temporal autocorrelation structure of signals. They treat
each time sample in the intermediate representations independently.
In this section, we introduce the \method layer that aligns the PSD of each signal onto a barycenter PSD within the architecture of a deep learning model.


{\method is a novel normalization layer that can be used as a drop-in replacement for layers like BatchNorm or InstanceNorm.
Instead of simple standardization, it aligns the Power Spectral Density (PSD) of feature maps to a running barycenter PSD.
This approach, optimized for modern hardware, enhances model robustness to new subjects at inference time without retraining.}
We define the normalized feature map as $\widetilde{\bG} \triangleq \text{\method}(\bG)$. The following sections introduce the core components of \method and its implementation.

\subsection{Core Components of the layer}

In the following, we formally define \method and present each of its three main
components: 1) PSD estimation,  2) running Riemannian barycenter estimation, and 3) $\filtersize$-Monge mapping computation.
Given a batch $\cB = \{\bG^{(1)}, \dots, \bG^{(N)}\}$ of $N$ pre-normalization
feature maps, \method outputs a normalized batch $\widetilde{\cB} =
\{\widetilde{\bG}^{(1)}, \dots, \widetilde{\bG}^{(N)}\}$ with normalized PSD.
Those three steps are detailed in the following and
illustrated in the right part of Figure~\ref{fig:concept}.
\paragraph{PSD Estimation}
\begin{wrapfigure}{r}{0.45\textwidth}
    \centering
    \vspace{-1.cm}
    \includegraphics[width=0.45\textwidth]{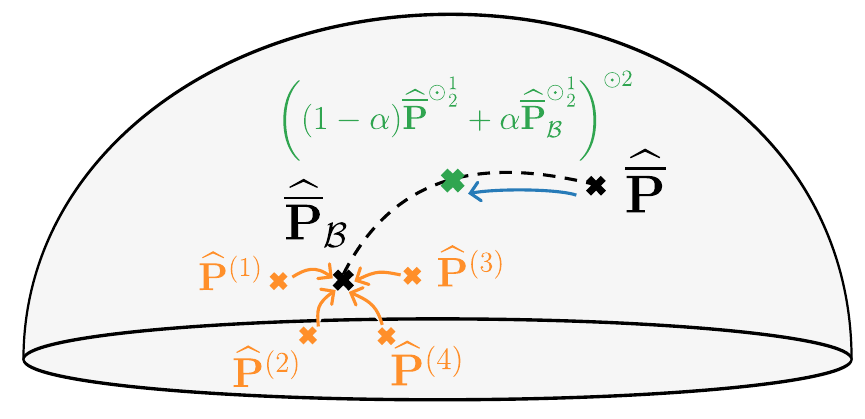}
    \caption{\textbf{Description of the running Riemanian barycenter.}
    The barycenter of the batch $\widehat{\overline{\bP}}_\cB$ is estimated from the PSD of each batch sample.
    }
    \vspace{-1cm}
    \label{fig:geodesic}
\end{wrapfigure}

First, the estimation of the PSD of each feature map is performed using the Welch method.
The per-channel mean $\widehat{\bmu}^{(j)}$ is computed for each feature map $\bG^{(j)}$ as
    $\widehat{\bmu}^{(j)} \triangleq \frac{1}{\ell}\sum_{l=1}^{\ell} \left[\bG^{(j)}\right]_{:,l} \in \bbR^{c} \;.$

Then, the PSD of the centered feature map $\bG^{(j)} -
\widehat{\bmu}^{(j)}\bone_\ell^\top$, denoted $\widehat{\bP}^{(j)}$, is estimated as described in Equation~\eqref{eq:welch}.
This centering step is required as feature maps are typically non-centered due to activation functions and convolution biases but they are assumed to have a stationary mean.
The Welch estimation involves segmenting the centered feature map into overlapping windows, computing the Fourier transform of each window and then averaging them.

\paragraph{Geodesic and Running Riemanian Barycenter}

The \method aligns the PSD of each feature map to a barycenter PSD.
This barycenter is computed during training by interpolating between the batch Wasserstein barycenter and the current running Riemanian barycenter using the geodesic associated with the Bures metric~\citep{bhatia2019bures}.
The batch barycenter is first computed from the current batch PSDs $\big\{\widehat{\bP}^{(1)}, \dots, \widehat{\bP}^{(N)}\big\}$ using \cref{eq:barycenter}.
To ensure gradual adaptation, the running barycenter is updated via an exponential geodesic average with $\alpha \in [0,1]$:
\begin{equation}
    \label{eq:running_barycenter_method}
    \widehat{\overline{\bP}} \leftarrow 
    \left((1-\alpha)\widehat{\overline{\bP}}^{\odot\tfrac12}
    + \alpha {\widehat{\overline{\bP}}_\cB}^{\odot\tfrac12}\right)^{\odot 2} \in \bbR^{c \times \filtersize} \;.
\end{equation}
A proof of the geodesic is provided in Appendix~\ref{app:geodesic}.

\paragraph{PSD Adaptation with$\filtersize$-Monge Mapping}

The final step of the \method is the application of the $\filtersize$-Monge mapping to each feature map after subtracting the per-channel mean.
Indeed, for all $j \in \intset{N}$, it is defined as
\begin{equation}
    \label{eq:method}
    \begin{aligned}
        \widetilde{\bG}^{(j)} &= m_\filtersize\left(\bG^{(j)} - \widehat{\bmu}^{(j)}\bone_\ell^\top, {\widehat{\overline{\bP}}} \right) 
        = \left(\left(\bG^{(j)} - \widehat{\bmu}^{(j)}\bone_\ell^\top \right) * \widehat{\bH}^{(j)}\right) \in \bbR^{c \times \ell}
    \end{aligned}
\end{equation}
where $\widehat{\bH}^{(j)}$ is the Monge mapping filter computed as
\begin{equation}
    \label{eq:monge_filter_method}
    \widehat{\bH}^{(j)} \triangleq \frac{1}{\sqrt{\filtersize}} \left( {\widehat{\overline{\bP}}} \oslash \widehat{\bP}^{(j)} \right)^{\odot{\frac{1}{2}}} \bF_{\filtersize}^* \in \bbR^{c\times f}
\end{equation}
where $\widehat{\bP}^{(j)}$ is the estimated PSD of $\bG^{(j)} - \widehat{\bmu}^{(j)}\bone_\ell^\top$.

\subsection{Implementation details}

\paragraph{Overall Algorithm}

The forward computation of the proposed layer is outlined in~\cref{alg:method}.
At train time, the \method performs three main operations: 1) PSD estimation, 2) running Riemannian barycenter update, and 3) Monge mapping application.
At inference, the \method operates similarly, except it does not update the running barycenter.
The \method is fully differentiable and can be integrated into any deep learning model.
Similarly to classical normalization layers, a stop gradient operation is
applied to the running barycenter to prevent the backpropagation of the gradient
computation through the barycenter. 
\method has a unique additional hyperparameter $\filtersize$ which is the filter size.
{It controls the alignment between each feature map and the running barycenter PSD and it is typically chosen in our experiments between 1 and 17.}
In practice, the Fourier transforms are efficiently computed using the Fast Fourier Transform (FFT) algorithm.
Because of the estimation of PSDs, the complexity of the \method, both at train
and inference times, is $\cO(N c \ell \filtersize \log(\filtersize))$, where $N$ is
the batch size, $c$ the number of channels, $\ell$ the signal length, and
$\filtersize$ the filter size.

\subsection{Discussion and Connections to Related Methods}

\begin{wrapfigure}{r}{0.6\textwidth}
    \vspace{-1.cm}
     \begin{minipage}{0.58\textwidth}
    \begin{algorithm}[H]
        \caption{Forward pass of \method}
        \label{alg:method}
        \begin{algorithmic}[1]
            \STATE {\textbf{Input:} Batch $\mathcal{B} = \left\{ \mathbf{G}^{(1)}, \dots, \mathbf{G}^{(N)} \right\}$, 
            running barycenter $\widehat{\overline{\mathbf{P}}}$, filter-size $\filtersize$, 
            momentum $\alpha$, training flag}
            
            \STATE {\textbf{Output:} Normalized batch $\left\{\widetilde{\mathbf{G}}^{(1)}, \dots, \widetilde{\mathbf{G}}^{(N)}\right\}$}

            \FOR{$j = 1$ to $N$}
                \STATE $\widehat{\bmu}^{(j)} \leftarrow$ Mean estimation
                \STATE $\widehat{\mathbf{P}}^{(j)} \leftarrow$ PSD est. from $\widetilde{\mathbf{G}}^{(j)}-\widehat{\bmu}^{(j)}\bone_\ell^\top$ with eq.~\eqref{eq:welch}
            \ENDFOR

            \IF{training}
                \STATE $\widehat{\overline{\mathbf{P}}}_\mathcal{B} \leftarrow$
                Batch bary. from $\{\widehat{\mathbf{P}}^{(j)}\}_j$ with eq.~\eqref{eq:barycenter}
                \STATE $\widehat{\overline{\mathbf{P}}} \leftarrow$ Running bary.
                up. from $\widehat{\overline{\mathbf{P}}},\widehat{\overline{\mathbf{P}}}_\mathcal{B}$ with eq.~\eqref{eq:running_barycenter_method}
            \ENDIF

            \FOR{$j = 1$ to $N$}
                \STATE $\widehat{\mathbf{H}}^{(j)} \leftarrow$ Filter estimation from $\widehat{\mathbf{P}}^{(j)}, \widehat{\overline{\mathbf{P}}}$ with eq.~\eqref{eq:monge_filter_method}
                \STATE $\widetilde{\mathbf{G}}^{(j)} \leftarrow$ $\filtersize$-Monge mapping with eq.~\eqref{eq:method}
            \ENDFOR
        \end{algorithmic}
    \end{algorithm}
\end{minipage}

\end{wrapfigure}

{\paragraph{\method as a generalization of InstanceNorm}
InstanceNorm applies a per-channel $z$-score over time, subtracting the mean and dividing by the standard deviation—equivalent to whitening under an i.i.d.\ assumption over time.
In contrast, PSDNorm explicitly accounts for temporal structure by estimating the PSD and whitening/re-coloring in the frequency domain.
InstanceNorm is recovered as a special case of PSDNorm by setting the filter size to $\filtersize = 1$ and using the uniform PSD barycenter as $\widehat{\overline{\bP}} = {\mathbf{1}}$.
as the re-coloring transform instead of the barycentric PSD.

\vspace{0.9cm}

\paragraph{Similarity with Test-time Domain Adaptation}

\method is inspired by Temporal Monge Alignment (TMA)~\citep{gnassounou2023convolutional}, a pre-processing technique that can be used for test-time adaptation. 
Test-time Domain Adaptation methods adjust a pre-trained model to a new target domain during inference, without requiring access to the original training data~\citep{wang_tent_2021, yang_generalized_2021}. 
While \method must be integrated into the model during training and is not a post-hoc adaptation method, it provides a similar benefit at inference time. 
Designing new modern architectures that incorporate \method can enhance robustness to domain shifts without the need for retraining or access to source data.

\paragraph{Discussion of Gaussian and Stationarity Assumptions}
\method relies on the Gaussian approximation of OT for compensation variability
but does not assume that the signals are Gaussian. This allows for efficient alignment of second-order statistics (covariance
structure), but also allows preserving higher-order discriminative information.
This approach is
computationally tractable and targets the most prominent sources of domain shift
without over-distorting the signal, a strategy also used in successful methods
like Deep CORAL~\citep{sun_deep_2016}. Like BatchNorm, \method assumes shifts
are captured by low-order statistics, but it provides a richer alignment by
incorporating temporal context.
}
\section{Numerical Experiments}
\label{sec:numerical_results}

In this section, we evaluate the proposed method through a series of experiments designed to highlight its effectiveness and robustness on the clinically relevant task of sleep staging.
We first describe the datasets and training setup employed, followed by a performance comparison with existing normalization techniques.
{Next, we assess the efficiency of \method by training over varying numbers of subjects per dataset.
Finally, we analyze the robustness of \method against domain shift by focusing on subject-wise performance and different architectures.}
The code is available at \url{https://github.com/tgnassou/PSDNorm}. The anonymized code is available in the supplementary material.
{All numerical experiments were conducted using a total of 1500 GPU hours on NVIDIA H100 GPUs.}

\subsection{Experimental Setup}

\paragraph{Datasets}
\begin{wraptable}{r}{0.5\textwidth}\vspace{-1.3cm}
     \centering
    \caption{\textbf{Characteristics of the datasets.}}
    \label{tab:datasets}
    \resizebox{\linewidth}{!}{
        \begin{tabular}{lcccc}
            \toprule
            Dataset & Subj. & Rec. & Age $\pm$ std & Sex (F/M) \\
            \midrule
            ABC & 44 & 117 & 48.8 $\pm$ 9.8 & 43\%/57\% \\
            CCSHS & 515 & 515 & 17.7 $\pm$ 0.4 & 50\%/50\% \\
            CFS & 681 & 681 & 41.7 $\pm$ 20.0 & 55\%/45\% \\
            HPAP & 166 & 166 & 46.5 $\pm$ 11.9 & 43\%/57\% \\
            MROS & 2101 & 2698 & 76.4 $\pm$ 5.5 & 0\%/100\% \\
            PHYS & 70 & 132 & 58.8 $\pm$ 22.0 & 33\%/67\% \\
            SHHS & 5730 & 8271 & 63.1 $\pm$ 11.2 & 52\%/48\% \\
            MASS & 61 & 61 & 42.5 $\pm$ 18.9 & 55\%/45\% \\
            CHAT & 1230 & 1635 & 6.6 $\pm$ 1.4 & 52\%/48\% \\
            SOF & 434 & 434 & 82.8 $\pm$ 3.1 & 100\%/0\% \\
            \midrule
            Total & 11032 & 14710 &  -- & -- \\
            \bottomrule
        \end{tabular}
        }
    \vspace{-0.4cm}
\end{wraptable}
To evaluate the effect of normalization layers, we use ten datasets of sleep staging described in~\cref{tab:datasets}.
ABC~\citep{jp2018ABC}, CCSHS~\citep{rosen_prevalence_2003ccshs}, 
CFS~\citep{redline_familial_1995cfs}, HPAP~\citep{rosen2012homepap}, MROS~\citep{blackwell_associations_2011Mros}, SHHS~\citep{SHHS}, CHAT~\citep{marcus2013CHAT}, and SOF~\citep{spira_sleep-disordered_2008SOF}
are publicly available sleep datasets with restricted access
from National Sleep Research Resource (NSRR)~\citep{zhang_national_2018}.
PHYS~\citep{Phys} and MASS~\citep{MASS} are two other datasets publicly available.
Every 30\,s epoch is labeled with one of the five sleep stages: Wake, N1, N2, N3, and REM.
These datasets are unbalanced in terms of age, sex, number of subjects, and have been recorded with different sensors in different institutions which makes the sleep staging task challenging.
We now describe the pre-processing steps and splits of the datasets.

\paragraph{Data Pre-processing}
We follow a standard pre-processing pipeline used in the field~\citep{chambon2017deep, Stephansen_2018}.
The datasets vary in the number and type of available EEG and electrooculogram (EOG) channels.
To ensure consistency, we use two bipolar EEG channels, as some datasets lack additional channels.
For dataset from NSRR, we select the channels C3-A2 and C4-A1.
For signals from Physionet and MASS, we use the only available channels Fpz-Cz and Pz-Oz.
The EEG signals are low-pass filtered with a 30\,Hz cutoff frequency and resampled to 100\,Hz.
All data extraction and pre-processing steps are implemented using MNE-BIDS~\citep{Appelhoff2019} and MNE-Python~\citep{GramfortEtAl2013a}.


{
\paragraph{Leave-One-Dataset-Out (LODO) Setup and Balancing}
We evaluate model performance using a leave-one-dataset-out (LODO) protocol: in each fold, one dataset is held out for testing, and the model is trained on the union of the remaining datasets.
From the training data, 80\% of subjects are used for training and 20\% for validation, which is used for early stopping.
The full held-out dataset is used for testing. 
To assess performance in low-data regimes, we also evaluate a variant in which we subsample at most $N$ subjects per dataset, promoting balanced contributions across training sources.
We refer to this configuration as \textbf{balanced@$N$}, with $N$ ranging from 40 to 400.
The exact number of subjects per dataset in each case is listed in Appendix Table~\ref{tab:balanced_datasets}.
}

\begin{table}[t]
    \centering
    \caption{
        {
            \textbf{Balanced Accuracy (BACC) scores on the left-out datasets with USleep.}
            The top section reports results in the \textbf{large-scale} setting (using all available subjects), while the bottom section presents results in the \textbf{medium-scale} setting (balanced@400).
            For each row, the best score is highlighted in \textbf{bold}, and standard deviations reflect training variability across 3 random seeds.
            The mean BACC reports the average over all the subjects.
        }
    }
    \label{tab:results}
    \resizebox{0.95\textwidth}{!}{
    \begin{tabular}{llccccc}
    \toprule
    & Dataset & BatchNorm & LayerNorm & InstanceNorm & TMA & PSDNorm \\
    \midrule
        \multirow[t]{11}{*}{\rotatebox[origin=r]{90}{All subjects\;\;\;\;\;\;\;\;\;\;}} & ABC & $78.49_{\pm 0.42}$ & $77.94_{\pm 0.31}$ & $\mathbf{78.83_{\pm 0.59}}$ & $78.33_{\pm 0.12}$ & $78.56_{\pm 0.67}$ \\
         & CCSHS & {$\mathbf{88.79_{\pm 0.21}}$} & $87.51_{\pm 0.77}$ & $88.75_{\pm 0.04}$ & $88.61_{\pm 0.10}$ & $88.56_{\pm 0.36}$ \\
         & CFS & $84.97_{\pm 0.37}$ & $84.29_{\pm 0.67}$ & $\mathbf{85.73_{\pm 0.29}}$ & $84.85_{\pm 0.13}$ & $85.42_{\pm 0.09}$ \\
         & CHAT & $64.72_{\pm 3.94}$ & $64.36_{\pm 0.40}$ & $68.86_{\pm 2.49}$ & $69.76_{\pm 1.62}$ & $\mathbf{70.57_{\pm 1.24}}$ \\
         & HOMEPAP & $76.39_{\pm 0.29}$ & $75.23_{\pm 0.78}$ & $76.70_{\pm 0.35}$ & $\mathbf{76.77_{\pm 0.66}}$ & $76.72_{\pm 0.27}$ \\
         & MASS & $73.71_{\pm 0.62}$ & $71.39_{\pm 3.00}$ & $72.12_{\pm 0.70}$ & $\mathbf{73.90_{\pm 0.69}}$ & $72.51_{\pm 1.68}$ \\
         & MROS & $81.30_{\pm 0.25}$ & $80.44_{\pm 0.29}$ & $81.49_{\pm 0.18}$ & $80.91_{\pm 0.42}$ & $\mathbf{81.57_{\pm 0.34}}$ \\
         & PhysioNet & $76.13_{\pm 0.57}$ & $75.12_{\pm 0.22}$ & $76.15_{\pm 0.52}$ & $\mathbf{76.48_{\pm 0.37}}$ & $75.96_{\pm 1.02}$ \\
         & SHHS & $77.97_{\pm 1.46}$ & $75.98_{\pm 0.48}$ & $79.05_{\pm 0.89}$ & $78.21_{\pm 0.39}$ & $\mathbf{79.14_{\pm 1.01}}$ \\
         & SOF & $81.33_{\pm 0.54}$ & $81.82_{\pm 0.79}$ & $81.98_{\pm 0.22}$ & $81.84_{\pm 0.49}$ & $\mathbf{82.50_{\pm 0.34}}$ \\
         \midrule
          & Mean(Dataset) & $78.38_{\pm 0.47}$ & $77.41_{\pm 0.28}$ & $78.97_{\pm 0.11}$ & $78.98_{\pm 0.14}$ & {$\mathbf{79.15_{\pm 0.14}}$} \\
          & Mean(Subject) & $78.14_{\pm 1.01}$ & $76.78_{\pm 0.18}$ & $79.26_{\pm 0.48}$ & $78.77_{\pm 0.07}$ & $\mathbf{79.51_{\pm 0.62}}$ \\
         \midrule
        \multirow[t]{11}{*}{\rotatebox[origin=r]{90}{Balanced@400 \;\;\;\;\;\;\;\;\;}} & ABC & $78.26_{\pm 1.33}$ & $75.29_{\pm 0.81}$ & $\mathbf{78.73_{\pm 0.42}}$ & $78.04_{\pm 0.51}$ & $78.18_{\pm 0.68}$ \\
         & CCSHS & $87.42_{\pm 0.16}$ & $85.20_{\pm 0.48}$ & $\mathbf{87.62_{\pm 0.42}}$ & $87.57_{\pm 0.20}$ & $87.58_{\pm 0.30}$ \\
         & CFS & $84.32_{\pm 0.57}$ & $81.66_{\pm 1.36}$ & $\mathbf{84.72_{\pm 0.33}}$ & $84.58_{\pm 0.20}$ & $84.29_{\pm 0.36}$ \\
         & CHAT & $66.55_{\pm 0.88}$ & $61.19_{\pm 1.16}$ & $64.43_{\pm 4.41}$ & $68.73_{\pm 2.48}$ & $\mathbf{70.28_{\pm 1.70}}$ \\
         & HOMEPAP & $75.25_{\pm 0.50}$ & $74.86_{\pm 0.25}$ & $76.47_{\pm 0.63}$ & $76.10_{\pm 0.32}$ & $\mathbf{76.83_{\pm 0.61}}$ \\
         & MASS & $70.00_{\pm 1.91}$ & $68.56_{\pm 3.33}$ & $71.52_{\pm 1.13}$ & $71.63_{\pm 1.92}$ & $\mathbf{72.77_{\pm 1.09}}$ \\
         & MROS & $\mathbf{80.37_{\pm 0.20}}$ & $78.05_{\pm 0.22}$ & $80.28_{\pm 0.21}$ & $80.09_{\pm 0.40}$ & $80.26_{\pm 0.11}$ \\
         & PhysioNet & $\mathbf{75.81_{\pm 0.13}}$ & $71.82_{\pm 2.12}$ & $74.68_{\pm 0.55}$ & $75.31_{\pm 1.54}$ & $74.82_{\pm 2.11}$ \\
         & SHHS & $76.44_{\pm 0.92}$ & $75.12_{\pm 0.39}$ & $78.68_{\pm 0.37}$ & $77.00_{\pm 0.39}$ & $\mathbf{78.88_{\pm 0.68}}$ \\
         & SOF & $81.08_{\pm 1.14}$ & $78.70_{\pm 0.50}$ & $80.68_{\pm 1.38}$ & $\mathbf{81.25_{\pm 0.71}}$ & $79.49_{\pm 0.41}$ \\
         \midrule
          & Mean(Dataset)  &  $77.55_{\pm 0.34}$ & $75.05_{\pm 0.28}$ & $77.78_{\pm 0.46}$ & $78.03_{\pm 0.35}$ & {$\mathbf{78.34_{\pm 0.42}}$} \\
         & Mean(Subject) & $77.22_{\pm 0.34}$ & $75.04_{\pm 0.42}$ & $78.17_{\pm 0.28}$ & $77.74_{\pm 0.36}$ & $\mathbf{78.85_{\pm 0.59}}$ \\
        \bottomrule
        \end{tabular}
        }
\end{table}


{
    \paragraph{Architecture and Training}
    Sleep staging has inspired a variety of neural architectures, from early CNN-based models~\citep{chambon2017deep, Stephansen_2018, xsleepnet} to recent attention-based approaches~\citep{phan2022sleeptransformer, phan_l-seqsleepnet_2023, wang_caresleepnet_2024}.
    We evaluate two architectures: \textbf{U-Sleep}~\citep{perslev2019u, Perslev2021USleep}, a state-of-the-art temporal CNN model designed for robustness and large-scale training, and a newly introduced architecture, \textbf{CNNTransformer}.
    CNNTransformer combines a lightweight convolutional encoder with a Transformer applied to epoch-level embeddings.
    It is specifically tailored for two-channel EEG and designed to scale efficiently to large datasets, while remaining minimal in implementation (under 100 lines of code) and training cost (\Cref{sec:cnn_transformer}).
    Its design draws inspiration from recent transformer-based models for time series~\citep{yang2023bios}, with an emphasis on simplicity and practicality.
}



{
We use the Adam optimizer~\citep{kingma2014adam} with a learning rate of $10^{-3}$ to minimize the weighted cross-entropy loss, where class weights are computed from the training set distribution. 
Training is performed with a batch size of 64, and early stopping is applied based on validation loss with a patience of 3 epochs.
Each input corresponds to a sequence of 17'30s, with a stride of 10'30s between sequences along the full-night recording.
}
{The filter size $\filtersize$ of \method is set to 5. A sensitivity analysis of $\filtersize$ is provided in \Cref{app:filter_size} in the appendix, and show that 
the performance is stable across a range of values from 5 to 11.}

\paragraph{Evaluation}
At inference, the model similarly processes sequences of 17'30s with a stride of 10'30s.
Performance is evaluated using the balanced accuracy score (BACC), computed on the central 10'30s of each prediction window. 
Each experiment is repeated three times with different random seeds, and we report the mean and standard deviation of BACC.


\paragraph{Normalization Strategies}
We compare the proposed \method with three normalization strategies: BatchNorm, LayerNorm, and InstanceNorm.
Note that InstanceNorm corresponds to a special case of \method with $\filtersize=1$ and a fixed identity mapping instead of a learned running barycenter.
In the following experiments, the BatchNorm layers in the first three convolutional layers are replaced with either PyTorch's default implementations of LayerNorm, InstanceNorm~\citep{pytorch}, or \method.
To preserve the receptive field, the filter size $\filtersize$ of \method is used in the first layer and progressively halved in the following ones.
We fix the momentum $\alpha$ to $10^{-2}$.

\subsection{Numerical Results}
{
    This section presents results from large-scale sleep stage classification experiments. 
    The analysis begins with a comparison of \method against standard normalization layers—BatchNorm, LayerNorm, and InstanceNorm—on the full datasets. 
    Then, the data efficiency of each method is evaluated under limited training data regimes. 
    Finally, robustness to distribution shift is assessed via subject-wise performance across multiple neural network architectures.
}


\paragraph{Performance Comparison on Full Datasets}

{
    {\Cref{tab:results} (top) reports the LODO BACC of U-Sleep across all datasets, averaged over three random seeds. 
    \method consistently outperforms all baseline normalization layers—BatchNorm, LayerNorm, InstanceNorm, and TMA—achieving the highest mean BACC of 79.51\% over subjects, which exceeds BatchNorm (78.38\%), InstanceNorm (78.97\%), LayerNorm (77.41\%) and TMA (78.77\%).
    On the challenging CHAT dataset, where all methods struggle, \method outperforms all other normalizations by more than 1 percentage points, highlighting its robustness under strong distribution shifts.
    Although InstanceNorm is a strong baseline—outperforming BatchNorm and LayerNorm by at least one standard deviation on average—it is consistently surpassed by \method in average performance.
    In contrast, LayerNorm underperforms across the board, achieving the lowest average BACC and never ranking first, confirming its limited suitability for this task.
    \method also improves score by almost 1\% over TMA, showing that using Monge Alignment inside the network allows for better adaptation.}
}


\paragraph{Efficiency: Performance with 4$\times$ Less Data}
\begin{wrapfigure}{r}{0.52\textwidth}
    \vspace{-0.5cm}
    \centering
    \includegraphics[width=\linewidth]{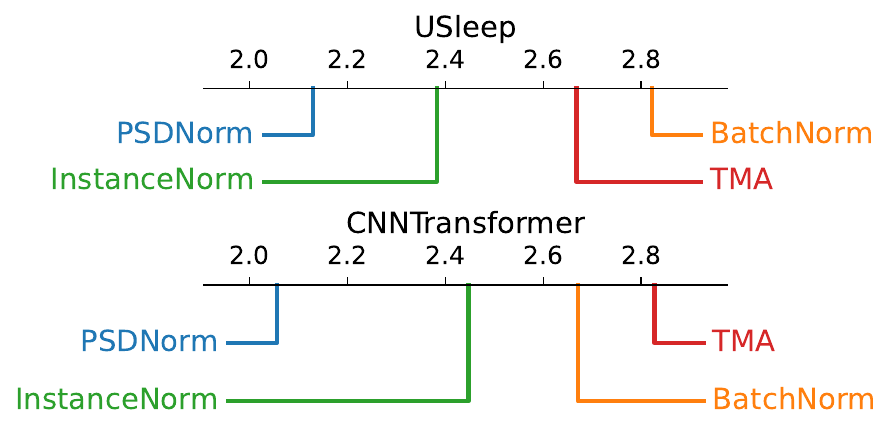}
    \caption{
        \textbf{Critical Difference (CD) diagram for two architectures on datasets balanced @400.}
        Average ranks across datasets and subjects for USleep and CNNTransformer.
        Black lines connect methods that are not significantly different.
    }
    \label{fig:cd_diagram}
\end{wrapfigure}
    The \method{} {layer improves model performance when trained on the full dataset ($\sim$10000 subjects), but such large-scale data availability is not always the case. 
    In many real-world scenarios—such as rare disease studies, pediatric populations, or data collected in constrained clinical settings—labeled recordings are scarce, expensive to annotate, or restricted due to privacy concerns. 
    Evaluating model robustness under these constraints is therefore essential.
    To this end, we train all models using the balanced@400 setup, which reduces the training data by a factor of 4 compared to the full-data setting.
    In this lower-data regime, \method continues to outperform all baseline normalization strategies and achieves higher average BACC. 
    {The performance improvement of \method over the best baseline is more pronounced in this setting: the BACC gain reaches $+0.67\%$, compared to $+0.25\%$ in the full-data setting.
    The gains exceed one standard deviation.}
    To assess statistical significance, we conducted a critical difference (CD) test~\citep{demvsar2006statistical}. 
    \Cref{fig:cd_diagram} (top) reports the average rank of each method and the corresponding statistical comparisons.
    The results confirm that \method significantly outperforms the baselines, underscoring the value of incorporating temporal structure into normalization for robust and data-efficient generalization.
    The same trend is observed for U-Sleep trained on all subjects (see in
    Appendix \Cref{fig:critical_diagram_all_subjects}). 
    The following experiments focus on the balanced@400 setup.
}

\paragraph{Robustness Across Architectures}

{
\method is a plug-and-play normalization layer that can be seamlessly integrated into various neural network architectures. 
To demonstrate this flexibility, we evaluate its performance on both the U-Sleep and CNNTransformer models.
\Cref{fig:cd_diagram} reports the average rank of each normalization method across datasets and subjects for both architectures using datasets balanced@400.
In both architectures, \method achieves the best overall ranking and demonstrates statistically significant improvements over both BatchNorm, InstanceNorm, and TMA.
{The results confirm that \method generalizes well beyond a single architecture and can provide consistent improvements in diverse modeling setups which is not the case of TMA
that is ranked the worst with CNNTransformer.}
InstanceNorm performs competitively in some cases but is never significantly better than \method.
Detailed numerical scores for CNNTransformer are reported in the supplementary material (\cref{tab:cnntransformer}).

{It is important to highlight that \method brings improvements without too much additional computational cost.
In appendix \Cref{app:psdnorm_time} we provide a detailed comparison of the computational time of \method with other normalization layers. The results show that \method is only slightly slower than BatchNorm and InstanceNorm, with a negligible increase in training time (less than 10\%) and no significant impact on inference speed.}
}

\newpage
\paragraph{Performance on the most challenging subjects}
\begin{wrapfigure}{r}{0.48\textwidth}
    \centering
    \includegraphics[width=.8\linewidth]{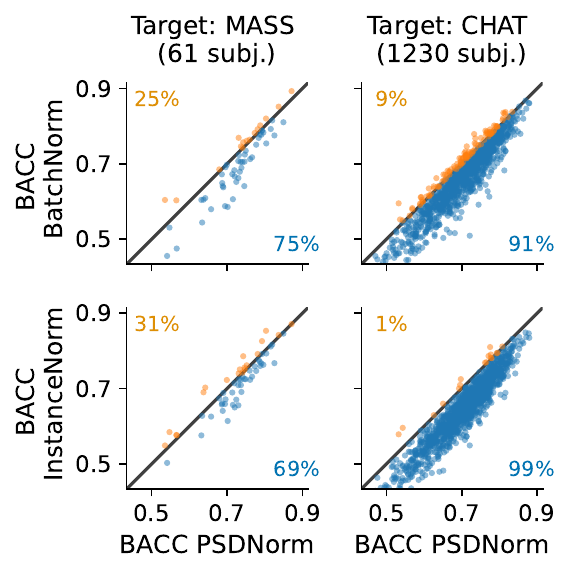}\vspace{-2mm}
    \caption{
        \textbf{Subject-wise BACC comparison on MASS and CHAT (balanced @400).}
        Blue dot means improvement with \method.
    }
    \label{fig:scatter}
    \vspace{-0.3cm}
\end{wrapfigure}
Performance variability across subjects is a key challenge in biomedical 
applications where ensuring consistently high performance—even for the most challenging subjects—is critical.
To highlight the robustness of \method,
\Cref{fig:scatter} presents a scatter 
plot of subject-wise BACC scores comparing BatchNorm or InstanceNorm vs. \method across two selected target datasets.
{CHAT and MASS are two challenging datasets,
where the prediction performance is significantly lower than the other datasets.
For CHAT, most of the dots are below the diagonal, indicating that \method improves performance 
for 91\% of subjects against BatchNorm and 99\% of subjects against InstanceNorm,
with the largest gains observed for the hardest subjects, reinforcing its ability to handle challenging cases.
For MASS, \method improves performance for 75\% of subjects against BatchNorm and 69\% against InstanceNorm.
This demonstrates that \method is not only effective in improving overall performance but also excels
in enhancing the performance of the most challenging subjects.
}


\subsection{Illustration of PSD Normalization}

Figure \ref{fig:psd_normalization} shows how different normalization layers affect the PSD of signals at several stages of the network.
The input signals display limited variability, which explains why applying TMA as a pre-processing step provides only marginal benefit.
In the first row, corresponding to BatchNorm, the PSD variability increases with depth, a behavior that is undesirable for generalization.
TMA exhibits a similar pattern, as no normalization is applied within the network to counteract this accumulation of variance.
In contrast, both InstanceNorm and \method reduce PSD variability across samples.
However, InstanceNorm does not fully align the PSDs, and noticeable differences remain between samples.
\method, on the other hand, achieves strong alignment of PSDs across samples, indicating its ability to normalize the underlying temporal correlations.
This alignment is essential for improving robustness and generalization, particularly in settings involving distribution shifts.

\begin{figure}[ht]
    \centering
    \includegraphics[width=0.8\linewidth]{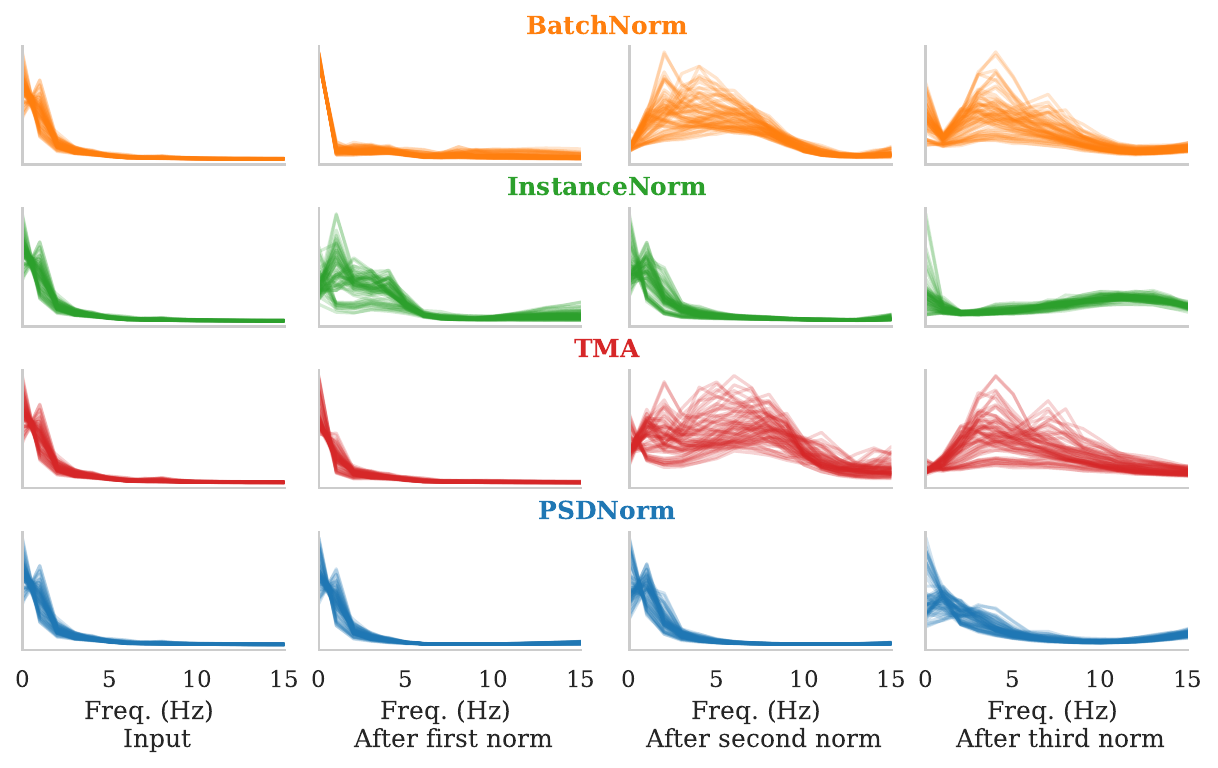}
    \caption{
        \textbf{Illustration of PSD normalization with different normalization layers.}
        The figure shows the PSD of  different segment of 17min from one subject batch as input, and after 3 encoders using different normalization layers.
    }
    \label{fig:psd_normalization}
\end{figure}

\color{black}
\section{Conclusion, Limitations, and Future Work}
\label{sec:conclusion}

{
This paper introduced \method, a normalization layer that aligns the power spectral density (PSD) of each signal to a geodesic barycenter.
By leveraging temporal correlations, \method offers a principled alternative to standard normalization layers.
Experiments on large-scale sleep staging datasets show that \method consistently improves performance, robustness, and data efficiency, especially under domain shift and limited-data settings—outperforming BatchNorm, LayerNorm, and InstanceNorm across architectures.
\\
While the results are promising, some limitations remain.
\method introduces a filter size hyperparameter ($\filtersize$) that controls normalization strength; although we provide default values that perform well across datasets, selecting it automatically in adaptive settings could be challenging.
\\
Despite these limitations, \method is flexible and easy to integrate into existing models.
Future work includes extending it to other signals such as audio and other biomedical applications.
}

\section*{Acknowledgements}
\noindent
This work was supported by the grants ANR-22-PESN-0012 to AC under the France 2030 program, 
ANR-20-CHIA-0016 and ANR-20-IADJ-0002 to AG while at Inria, and
ANR-23-ERCC-0006, ANR-25-PEIA-0005 and ANR-23-IACL-0005 to RF, 
all from Agence nationale de la recherche (ANR). This work is supported by Hi! PARIS and ANR/France 2030 program
(ANR-23-IACL-0005).
This project has also received funding from the European Union’s Horizon  Europe  research  
and  innovation  programme  under  grant  agreement  101120237 (ELIAS).

This project received funding from the Fondation de l'École polytechnique

This project was provided with computer and storage resources by GENCI at
IDRIS thanks to the grant 2025-AD011016052 and 2025-AD011016067 on the supercomputer
Jean Zay's the V100 \& H100 partitions.

This work was conducted at Inria, AG is presently employed by Meta Platforms. All the datasets used for this work were accessed and processed on the Inria compute infrastructure.

All the datasets used for this work were accessed and processed on the Inria compute infrastructures.
\\
\sloppy
Numerical computation was enabled by the 
scientific Python ecosystem: 
Matplotlib~\cite{matplotlib}, 
Scikit-learn~\cite{scikit-learn}, 
Numpy~\cite{numpy}, Scipy~\cite{scipy}, PyTorch~\cite{pytorch} and  MNE~\cite{GramfortEtAl2013a}.

\bibliography{biblio}

@article{Supratak_2017,
  title={{DeepSleepNet}: A model for automatic sleep stage scoring based on raw single-channel {EEG}},
  author={Supratak, Akara and Dong, Hao and Wu, Chao and Guo, Yike},
  journal={IEEE Transactions on Neural Systems and Rehabilitation Engineering},
  volume={25},
  number={11},
  pages={1998--2008},
  year={2017},
  publisher={IEEE}
}

@article{chambon2018deep,
  title={A deep learning architecture for temporal sleep stage classification using multivariate and multimodal time series},
  author={Chambon, Stanislas and Galtier, Mathieu N and Arnal, Pierrick J and Wainrib, Gilles and Gramfort, Alexandre},
  journal={IEEE Transactions on Neural Systems and Rehabilitation Engineering},
  volume={26},
  number={4},
  pages={758--769},
  year={2018},
  publisher={IEEE}
}

@article{MORENOTORRES2012521,
title = {A unifying view on dataset shift in classification},
journal = {Pattern Recognition},
volume = {45},
number = {1},
pages = {521-530},
year = {2012},
issn = {0031-3203},
doi = {https://doi.org/10.1016/j.patcog.2011.06.019},
author = {Jose G. Moreno-Torres and Troy Raeder and Rocío Alaiz-Rodríguez and Nitesh V. Chawla and Francisco Herrera},
}

@article{apicella2023effects,
  title={On the effects of data normalization for domain adaptation on {EEG} data},
  author={Apicella, Andrea and Isgr{\`o}, Francesco and Pollastro, Andrea and Prevete, Roberto},
  journal={Engineering Applications of Artificial Intelligence},
  volume={123},
  pages={106205},
  year={2023},
  publisher={Elsevier}
}

@misc{sun_deep_2016,
	title = {Deep {CORAL}: {Correlation} {Alignment} for {Deep} {Domain} {Adaptation}},
	shorttitle = {Deep {CORAL}},
	url = {http://arxiv.org/abs/1607.01719},
	urldate = {2022-09-19},
	publisher = {arXiv},
	author = {Sun, Baochen and Saenko, Kate},
	month = jul,
	year = {2016},
	note = {arXiv:1607.01719 [cs]},
}

@article{Perslev2021Usleep,
author = {Perslev, Mathias and Darkner, Sune and Kempfner, Lykke and Nikolic, Miki and Jennum, Poul and Igel, Christian},
year = {2021},
month = {04},
pages = {72},
title = {{U-Sleep}: resilient high-frequency sleep staging},
volume = {4},
journal = {npj Digital Medicine},
doi = {10.1038/s41746-021-00440-5}
}

@article{MASS,
author = {O'Reilly, Christian and Gosselin, Nadia and Carrier, Julie},
year = {2014},
month = {06},
pages = {},
title = {Montreal Archive of Sleep Studies: an open-access resource for instrument benchmarking and exploratory research},
volume = {23},
journal = {Journal of sleep research},
doi = {10.1111/jsr.12169}
}

@article{Phys,
author = {Goldberger, Ary and Amaral, Luís and Glass, L. and Havlin, Shlomo and Hausdorg, J. and Ivanov, Plamen and Mark, R. and Mietus, J. and Moody, G. and Peng, Chung-Kang and Stanley, H. and Physiobank, Physiotoolkit},
year = {2000},
month = {01},
pages = {},
title = {Components of a new research resource for complex physiologic signals},
volume = {101},
journal = {PhysioNet}
}

@article{SHHS,
author = {Quan, Stuart and Howard, Barbara and Iber, Conrad and Kiley, James and Nieto, F. and O'Connor, George and Rapoport, David and Redline, Susan and Robbins, John and Samet, Jonathan and Wahl, Patricia},
year = {1998},
month = {01},
pages = {1077-85},
title = {The Sleep Heart Health Study: Design, Rationale, and Methods},
volume = {20},
journal = {Sleep},
doi = {10.1093/sleep/20.12.1077}
}

@article{bhatia2019bures,
  title={On the Bures--Wasserstein distance between positive definite matrices},
  author={Bhatia, Rajendra and Jain, Tanvi and Lim, Yongdo},
  journal={Expositiones Mathematicae},
  volume={37},
  number={2},
  pages={165--191},
  year={2019},
  publisher={Elsevier}
}

@article{Stephansen_2018,
	doi = {10.1038/s41467-018-07229-3},  
	year = 2018,
	month = {dec},
  	publisher = {Springer Science and Business Media {LLC}
},
	volume = {9},
	number = {1}, 
	author = {Jens B. Stephansen and Alexander N. Olesen and Mads Olsen and Aditya Ambati and Eileen B. Leary and Hyatt E. Moore and Oscar Carrillo and Ling Lin and Fang Han and Han Yan and Yun L. Sun and Yves Dauvilliers and Sabine Scholz and Lucie Barateau and Birgit Hogl and Ambra Stefani and Seung Chul Hong and Tae Won Kim and Fabio Pizza and Giuseppe Plazzi and Stefano Vandi and Elena Antelmi and Dimitri Perrin and Samuel T. Kuna and Paula K. Schweitzer and Clete Kushida and Paul E. Peppard and Helge B. D. Sorensen and Poul Jennum and Emmanuel Mignot},
	title = {Neural network analysis of sleep stages enables efficient diagnosis of narcolepsy},
	journal = {Nature Communications}
}

@article{GramfortEtAl2013a,
  title = {{{MEG}} and {{EEG}} Data Analysis with {{MNE}}-{{Python}}},
  author = {Gramfort, Alexandre and Luessi, Martin and Larson, Eric and Engemann, Denis A. and Strohmeier, Daniel and Brodbeck, Christian and Goj, Roman and Jas, Mainak and Brooks, Teon and Parkkonen, Lauri and H{\"a}m{\"a}l{\"a}inen, Matti S.},
  year = {2013},
  volume = {7},
  pages = {1--13},
  doi = {10.3389/fnins.2013.00267},
  journal = {Frontiers in Neuroscience},
  number = {267}
}

@article{Appelhoff2019, doi = {10.21105/joss.01896}, year = {2019}, publisher = {The Open Journal}, volume = {4}, number = {44}, pages = {1896}, author = {Stefan Appelhoff and Matthew Sanderson and Teon L. Brooks and Marijn van Vliet and Romain Quentin and Chris Holdgraf and Maximilien Chaumon and Ezequiel Mikulan and Kambiz Tavabi and Richard Höchenberger and Dominik Welke and Clemens Brunner and Alexander P. Rockhill and Eric Larson and Alexandre Gramfort and Mainak Jas}, title = {{MNE-BIDS}: Organizing electrophysiological data into the {BIDS} format and facilitating their analysis}, journal = {Journal of Open Source Software} }

@ARTICLE{xsleepnet,
author = {Huy Phan and Oliver Y. Chen and Minh C. Tran and Philipp Koch and Alfred Mertins and Maarten De Vos},
journal = {IEEE Transactions on Pattern Analysis and Machine Intelligence},
title = {{XSleepNet}: Multi-View Sequential Model for Automatic Sleep Staging},
year = {2022},
volume = {44},
number = {09},
issn = {1939-3539},
pages = {5903-5915},
doi = {10.1109/TPAMI.2021.3070057},
publisher = {IEEE Computer Society},
address = {Los Alamitos, CA, USA},
month = {sep}
}

@incollection{STEVENS200445,
title = {CHAPTER 6 - POLYSOMNOGRAPHY},
editor = {DAMIEN STEVENS},
booktitle = {Sleep Medicine Secrets},
publisher = {Hanley \& Belfus},
pages = {45-63},
year = {2004},
isbn = {978-1-56053-592-8},
doi = {https://doi.org/10.1016/B978-1-56053-592-8.50010-5},
author = {Suzanne Stevens and Glenn Clark}
}

@article{guillot2021robustsleepnet,
      title={{RobustSleepNet}: Transfer learning for automated sleep staging at scale}, 
      author={Antoine Guillot and Valentin Thorey},
      year={2021},
      eprint={2101.02452},
      archivePrefix={arXiv},
      primaryClass={stat.ML}
}

@inproceedings{gnassounou2023convolutional,
author = {Gnassounou, Théo and Flamary, Rémi and Gramfort, Alexandre},
title = {Convolutional Monge Mapping Normalization for learning on biosignals},
booktitle = {Neural Information Processing Systems (NeurIPS)},
year = {2023}
}

@article{jp2018ABC,
	title = {Gastric {Banding} {Surgery} versus {Continuous} {Positive} {Airway} {Pressure} for {Obstructive} {Sleep} {Apnea}: {A} {Randomized} {Controlled} {Trial}},
	volume = {197},
	issn = {1535-4970},
	shorttitle = {Gastric {Banding} {Surgery} versus {Continuous} {Positive} {Airway} {Pressure} for {Obstructive} {Sleep} {Apnea}},
	url = {https://pubmed.ncbi.nlm.nih.gov/29035093/},
	doi = {10.1164/rccm.201708-1637LE},
	language = {en},
	number = {8},
	urldate = {2024-01-08},
	journal = {American journal of respiratory and critical care medicine},
	author = {Jessie P., Bakker and Ali, Tavakkoli and Michael, Rueschman and Wei, Wang and Robert, Andrews and Atul, Malhotra and Robert L., Owens and Amit, Anand and Katherine, Dudley and Sanya R., Patel},
	month = apr,
	year = {2018},
	pmid = {29035093},
	note = {Publisher: Am J Respir Crit Care Med},
}

@article{marcus2013CHAT,
	title = {A randomized trial of adenotonsillectomy for childhood sleep apnea},
	volume = {368},
	issn = {1533-4406},
	doi = {10.1056/NEJMoa1215881},
	language = {eng},
	number = {25},
	journal = {The New England Journal of Medicine},
	author = {Marcus, Carole L. and Moore, Reneé H. and Rosen, Carol L. and Giordani, Bruno and Garetz, Susan L. and Taylor, H. Gerry and Mitchell, Ron B. and Amin, Raouf and Katz, Eliot S. and Arens, Raanan and Paruthi, Shalini and Muzumdar, Hiren and Gozal, David and Thomas, Nina Hattiangadi and Ware, Janice and Beebe, Dean and Snyder, Karen and Elden, Lisa and Sprecher, Robert C. and Willging, Paul and Jones, Dwight and Bent, John P. and Hoban, Timothy and Chervin, Ronald D. and Ellenberg, Susan S. and Redline, Susan and {Childhood Adenotonsillectomy Trial (CHAT)}},
	month = jun,
	year = {2013},
	pmid = {23692173},
	pmcid = {PMC3756808},
	pages = {2366--2376},
}

@article{rosen2012homepap,
	title = {A multisite randomized trial of portable sleep studies and positive airway pressure autotitration versus laboratory-based polysomnography for the diagnosis and treatment of obstructive sleep apnea: the {HomePAP} study},
	volume = {35},
	issn = {1550-9109},
	shorttitle = {A multisite randomized trial of portable sleep studies and positive airway pressure autotitration versus laboratory-based polysomnography for the diagnosis and treatment of obstructive sleep apnea},
	doi = {10.5665/sleep.1870},
	language = {eng},
	number = {6},
	journal = {Sleep},
	author = {Rosen, Carol L. and Auckley, Dennis and Benca, Ruth and Foldvary-Schaefer, Nancy and Iber, Conrad and Kapur, Vishesh and Rueschman, Michael and Zee, Phyllis and Redline, Susan},
	month = jun,
	year = {2012},
	pmid = {22654195},
	pmcid = {PMC3353048},
	pages = {757--767},
}

@article{schirrmeister2017deep,
	title = {Deep learning with convolutional neural networks for {EEG} decoding and visualization},
	volume = {38},
	issn = {1065-9471, 1097-0193},
	url = {http://arxiv.org/abs/1703.05051},
	doi = {10.1002/hbm.23730},
	number = {11},
	urldate = {2024-01-10},
	journal = {Human Brain Mapping},
	author = {Schirrmeister, Robin Tibor and Springenberg, Jost Tobias and Fiederer, Lukas Dominique Josef and Glasstetter, Martin and Eggensperger, Katharina and Tangermann, Michael and Hutter, Frank and Burgard, Wolfram and Ball, Tonio},
	month = nov,
	year = {2017},
	note = {arXiv:1703.05051 [cs]},
	pages = {5391--5420},
}

@inproceedings{yang_generalized_2021,
	address = {Montreal, QC, Canada},
	title = {Generalized {Source}-free {Domain} {Adaptation}},
	isbn = {978-1-66542-812-5},
	url = {https://ieeexplore.ieee.org/document/9710764/},
	doi = {10.1109/ICCV48922.2021.00885},
	language = {en},
	urldate = {2024-01-24},
	booktitle = {2021 {IEEE}/{CVF} {International} {Conference} on {Computer} {Vision} ({ICCV})},
	publisher = {IEEE},
	author = {Yang, Shiqi and Wang, Yaxing and Van De Weijer, Joost and Herranz, Luis and Jui, Shangling},
	month = oct,
	year = {2021},
	pages = {8958--8967},
}

@misc{wang_tent_2021,
	title = {Tent: {Fully} {Test}-time {Adaptation} by {Entropy} {Minimization}},
	shorttitle = {Tent},
	url = {http://arxiv.org/abs/2006.10726},
	urldate = {2024-01-25},
	publisher = {arXiv},
	author = {Wang, Dequan and Shelhamer, Evan and Liu, Shaoteng and Olshausen, Bruno and Darrell, Trevor},
	month = mar,
	year = {2021},
	note = {arXiv:2006.10726 [cs, stat]},
}

@article{welch1967use,
  title={The use of fast Fourier transform for the estimation of power spectra: a method based on time averaging over short, modified periodograms},
  author={Welch, Peter},
  journal={IEEE Transactions on audio and electroacoustics},
  volume={15},
  number={2},
  pages={70--73},
  year={1967},
  publisher={IEEE}
}

@incollection{pytorch,
	title        = {PyTorch: An Imperative Style, High-Performance Deep Learning Library},
	author       = {A. Paszke and S. Gross and F. Massa and A. Lerer and J. Bradbury and G. Chanan and T. Killeen and Z. Lin and N. Gimelshein and L. Antiga and A. Desmaison and A. Kopf and E. Yang and Z. DeVito and M. Raison and A. Tejani and S. Chilamkurthy and B. Steiner and L. Fang and J. Bai and S. Chintala},
	year         = 2019,
	booktitle    = {Advances in Neural Information Processing Systems (NeurIPS)},
	publisher    = {Curran Associates, Inc.},
	pages        = {8024--8035}
}

@article{scipy,
	title        = {{{SciPy} 1.0: Fundamental Algorithms for Scientific Computing in Python}},
	author       = {P. Virtanen and R. Gommers and T.E. Oliphant and M. Haberland and T. Reddy and  D. Cournapeau and E. Burovski and P. Peterson and W. Weckesser and J. Bright and S.J. {van der Walt} and M. Brett and J. Wilson and J.K. Millman and N. Mayorov and A.R.J. Nelson and E. Jones and R. Kern and E. Larson and C.J. Carey and I. Polat and Y. Feng and E.W. Moore and J. {VanderPlas} and D. Laxalde and J. Perktold and R. Cimrman and I. Henriksen and E.A. Quintero and C.R. Harris and A.M. Archibald and A.H. Ribeiro and F. Pedregosa and P. {van Mulbregt} and {SciPy 1.0 Contributors}},
	year         = 2020,
	journal      = {Nature Methods},
	volume       = 17,
	pages        = {261--272}
}

@article{matplotlib,
	title        = {Matplotlib: A 2D graphics environment},
	author       = {John D. Hunter},
	year         = 2007,
	journal      = {Computing in science \& engineering},
	publisher    = {IEEE Computer Society},
	volume       = 9,
	number       = 3,
	pages        = {90--95}
}

@article{numpy,
	title        = {Array programming with {NumPy}},
	author       = {C.R. Harris and K.J. Millman and S.J. van der Walt and R. Gommers and P. Virtanen and D. Cournapeau and E. Wieser and J. Taylor and S. Berg and N.J. Smith and R. Kern and M. Picus and S. Hoyer and M.H. van Kerkwijk and M. Brett and A. Haldane and J. Fernández del Río and M. Wiebe and P. Peterson and P. G{'{e}}rard-Marchant and K. Sheppard and T. Reddy and W. Weckesser and H. Abbasi and C. Gohlke and T.E. Oliphant},
	year         = 2020,
	journal      = {Nature},
	publisher    = {Springer Science and Business Media {LLC}},
	volume       = 585,
	number       = 7825,
	pages        = {357--362}
}

@article{scikit-learn,
	title        = {{Scikit-learn: Machine Learning in Python }},
	author       = {Pedregosa, F. and Varoquaux, G. and Gramfort, A. and Michel, V. and Thirion, B. and Grisel, O. and Blondel, M. and Prettenhofer, P. and Weiss, R. and Dubourg, V. and Vanderplas, J. and Passos, A. and Cournapeau, D. and Brucher, M. and Perrot, M. and Duchesnay, E.},
	year         = 2011,
	journal      = {Journal of Machine Learning Research},
	volume       = 12,
	pages        = {2825--2830}
}

@article{gray06,
    url = {http://dx.doi.org/10.1561/0100000006},
    year = {2006},
    volume = {2},
    journal = {Foundations and Trends® in Communications and Information Theory},
    title = {Toeplitz and Circulant Matrices: A Review},
    doi = {10.1561/0100000006},
    issn = {1567-2190},
    number = {3},
    pages = {155-239},
    author = {Robert M. Gray}
}

@inproceedings{kobler-etal:22,
 author = {Kobler, Reinmar and Hirayama, Jun-ichiro and Zhao, Qibin and Kawanabe, Motoaki},
 booktitle = {Advances in Neural Information Processing Systems},
 editor = {S. Koyejo and S. Mohamed and A. Agarwal and D. Belgrave and K. Cho and A. Oh},
 pages = {6219--6235},
 publisher = {Curran Associates, Inc.},
 title = {SPD domain-specific batch normalization to crack interpretable unsupervised domain adaptation in {EEG}},
 volume = {35},
 year = {2022}
}

@article{perslev2019u,
  title={U-time: A fully convolutional network for time series segmentation applied to sleep staging},
  author={Perslev, Mathias and Jensen, Michael and Darkner, Sune and Jennum, Poul J{\o}rgen and Igel, Christian},
  journal={Advances in Neural Information Processing Systems},
  volume={32},
  year={2019}
}

@inproceedings{ioffe2015batch,
	author = {Ioffe, Sergey and Szegedy, Christian},
	title = {Batch normalization: accelerating deep network training by reducing internal covariate shift},
	year = {2015},
	publisher = {JMLR.org},
	booktitle = {Proceedings of the 32nd International Conference on International Conference on Machine Learning - Volume 37},
	pages = {448–456},
	numpages = {9},
	location = {Lille, France},
	series = {ICML'15}
}

@misc{ba2016layer,
	author = {Ba, Jimmy Lei and Kiros, Jamie Ryan and Hinton, Geoffrey E.},
	title = {Layer Normalization},
	url = {http://arxiv.org/abs/1607.06450},
	year = 2016
}

@article{ulyanov2016instance,
  title={Instance normalization: The missing ingredient for fast stylization},
  author={Ulyanov, D},
  journal={arXiv preprint arXiv:1607.08022},
  year={2016}
}

@article{gnassounou2024multi,
  title={Multi-Source and Test-Time Domain Adaptation on Multivariate Signals using Spatio-Temporal Monge Alignment},
  author={Gnassounou, Th{\'e}o and Collas, Antoine and Flamary, R{\'e}mi and Lounici, Karim and Gramfort, Alexandre},
  journal={arXiv preprint arXiv:2407.14303},
  year={2024}
}

@article{agueh2011barycenters,
  title={Barycenters in the Wasserstein space},
  author={Agueh, Martial and Carlier, Guillaume},
  journal={SIAM Journal on Mathematical Analysis},
  volume={43},
  number={2},
  pages={904--924},
  year={2011},
  publisher={SIAM}
}

@article{zhang_national_2018,
	title = {The {National} {Sleep} {Research} {Resource}: towards a sleep data commons},
	volume = {25},
	issn = {1527-974X},
	shorttitle = {The {National} {Sleep} {Research} {Resource}},
	doi = {10.1093/jamia/ocy064},
	language = {eng},
	number = {10},
	journal = {Journal of the American Medical Informatics Association: JAMIA},
	author = {Zhang, Guo-Qiang and Cui, Licong and Mueller, Remo and Tao, Shiqiang and Kim, Matthew and Rueschman, Michael and Mariani, Sara and Mobley, Daniel and Redline, Susan},
	month = oct,
	year = {2018},
	pmid = {29860441},
	pmcid = {PMC6188513},
	pages = {1351--1358},
}

@article{blackwell_associations_2011Mros,
	title = {Associations between sleep architecture and sleep-disordered breathing and cognition in older community-dwelling men: the {Osteoporotic} {Fractures} in {Men} {Sleep} {Study}},
	volume = {59},
	issn = {1532-5415},
	shorttitle = {Associations between sleep architecture and sleep-disordered breathing and cognition in older community-dwelling men},
	doi = {10.1111/j.1532-5415.2011.03731.x},
	language = {eng},
	number = {12},
	journal = {Journal of the American Geriatrics Society},
	author = {Blackwell, Terri and Yaffe, Kristine and Ancoli-Israel, Sonia and Redline, Susan and Ensrud, Kristine E. and Stefanick, Marcia L. and Laffan, Alison and Stone, Katie L. and {Osteoporotic Fractures in Men Study Group}},
	month = dec,
	year = {2011},
	pmid = {22188071},
	pmcid = {PMC3245643},
	pages = {2217--2225},
}

@article{redline_familial_1995cfs,
	title = {The familial aggregation of obstructive sleep apnea},
	volume = {151},
	issn = {1073-449X},
	doi = {10.1164/ajrccm/151.3_Pt_1.682},
	language = {eng},
	number = {3 Pt 1},
	journal = {American Journal of Respiratory and Critical Care Medicine},
	author = {Redline, S. and Tishler, P. V. and Tosteson, T. D. and Williamson, J. and Kump, K. and Browner, I. and Ferrette, V. and Krejci, P.},
	month = mar,
	year = {1995},
	pmid = {7881656},
	pages = {682--687},
}

@article{rosen_prevalence_2003ccshs,
	title = {Prevalence and risk factors for sleep-disordered breathing in 8- to 11-year-old children: association with race and prematurity},
	volume = {142},
	issn = {0022-3476},
	shorttitle = {Prevalence and risk factors for sleep-disordered breathing in 8- to 11-year-old children},
	doi = {10.1067/mpd.2003.28},
	language = {eng},
	number = {4},
	journal = {The Journal of Pediatrics},
	author = {Rosen, Carol L. and Larkin, Emma K. and Kirchner, H. Lester and Emancipator, Judith L. and Bivins, Sarah F. and Surovec, Susan A. and Martin, Richard J. and Redline, Susan},
	month = apr,
	year = {2003},
	pmid = {12712055},
	pages = {383--389},
}

@article{spira_sleep-disordered_2008SOF,
	title = {Sleep-disordered breathing and cognition in older women},
	volume = {56},
	issn = {1532-5415},
	doi = {10.1111/j.1532-5415.2007.01506.x},
	language = {eng},
	number = {1},
	journal = {Journal of the American Geriatrics Society},
	author = {Spira, Adam P. and Blackwell, Terri and Stone, Katie L. and Redline, Susan and Cauley, Jane A. and Ancoli-Israel, Sonia and Yaffe, Kristine},
	month = jan,
	year = {2008},
	pmid = {18047498},
	pages = {45--50},
}

@article{chambon2017deep,
      title={A deep learning architecture for temporal sleep stage classification using multivariate and multimodal time series}, 
      author={Stanislas Chambon and Mathieu Galtier and Pierrick Arnal and Gilles Wainrib and Alexandre Gramfort},
      year={2017},
      eprint={1707.03321},
      archivePrefix={arXiv},
      primaryClass={stat.ML}
}

@misc{phan_seqsleepnet_2019,
	title = {{SeqSleepNet}: {End}-to-{End} {Hierarchical} {Recurrent} {Neural} {Network} for {Sequence}-to-{Sequence} {Automatic} {Sleep} {Staging}},
	shorttitle = {{SeqSleepNet}},
	url = {http://arxiv.org/abs/1809.10932},
	urldate = {2024-10-21},
	publisher = {arXiv},
	author = {Phan, Huy and Andreotti, Fernando and Cooray, Navin and Chén, Oliver Y. and Vos, Maarten De},
	month = feb,
	year = {2019},
	note = {arXiv:1809.10932},
}

@article{phan_l-seqsleepnet_2023,
	title = {L-{SeqSleepNet}: {Whole}-cycle {Long} {Sequence} {Modeling} for {Automatic} {Sleep} {Staging}},
	volume = {27},
	issn = {2168-2208},
	shorttitle = {L-{SeqSleepNet}},
	url = {https://ieeexplore.ieee.org/document/10210638/?arnumber=10210638},
	doi = {10.1109/JBHI.2023.3303197},
	number = {10},
	urldate = {2024-10-21},
	journal = {IEEE Journal of Biomedical and Health Informatics},
	author = {Phan, Huy and Lorenzen, Kristian P. and Heremans, Elisabeth and Chén, Oliver Y. and Tran, Minh C. and Koch, Philipp and Mertins, Alfred and Baumert, Mathias and Mikkelsen, Kaare B. and De Vos, Maarten},
	month = oct,
	year = {2023},
	note = {Conference Name: IEEE Journal of Biomedical and Health Informatics},
	pages = {4748--4757},
}

@inproceedings{kim2021reversible,
  title={Reversible instance normalization for accurate time-series forecasting against distribution shift},
  author={Kim, Taesung and Kim, Jinhee and Tae, Yunwon and Park, Cheonbok and Choi, Jang-Ho and Choo, Jaegul},
  booktitle={International Conference on Learning Representations},
  year={2021}
}

@inproceedings{ronneberger2015u,
  title={U-net: Convolutional networks for biomedical image segmentation},
  author={Ronneberger, Olaf and Fischer, Philipp and Brox, Thomas},
  booktitle={Medical image computing and computer-assisted intervention--MICCAI 2015: 18th international conference, Munich, Germany, October 5-9, 2015, proceedings, part III 18},
  pages={234--241},
  year={2015},
  organization={Springer}
}

@article{kingma2014adam,
  title={Adam: A method for stochastic optimization},
  author={Kingma, Diederik P},
  journal={arXiv preprint arXiv:1412.6980},
  year={2014}
}

@article{wang_caresleepnet_2024,
	title = {{CareSleepNet}: A Hybrid Deep Learning Network for Automatic Sleep Staging},
	volume = {28},
	issn = {2168-2208},
	url = {https://ieeexplore.ieee.org/document/10595067/},
	doi = {10.1109/JBHI.2024.3426939},
	shorttitle = {{CareSleepNet}},
	pages = {7392--7405},
	number = {12},
	journaltitle = {{IEEE} Journal of Biomedical and Health Informatics},
	author = {Wang, Jiquan and Zhao, Sha and Jiang, Haiteng and Zhou, Yangxuan and Yu, Zhenghe and Li, Tao and Li, Shijian and Pan, Gang},
	urldate = {2025-04-14},
	date = {2024-12},
}

@article{thapa_multimodal_2025,
	title = {A Multimodal Sleep Foundation Model Developed with 500K Hours of Sleep Recordings for Disease Predictions},
	url = {https://www.ncbi.nlm.nih.gov/pmc/articles/PMC11838666/},
	doi = {10.1101/2025.02.04.25321675},
	pages = {2025.02.04.25321675},
	journaltitle = {{medRxiv}},
	shortjournal = {{medRxiv}},
	author = {Thapa, Rahul and Kjær, Magnus Ruud and He, Bryan and Covert, Ian and Moore, Hyatt and Hanif, Umaer and Ganjoo, Gauri and Westover, M. Brandon and Jennum, Poul and Brink-Kjær, Andreas and Mignot, Emmanuel and Zou, James},
	urldate = {2025-05-07},
	date = {2025-02-09},
	pmid = {39974074},
	pmcid = {PMC11838666},
}

@article{fox_foundational_2025,
	title = {A foundational transformer leveraging full night, multichannel sleep study data accurately classifies sleep stages},
	issn = {1550-9109},
	doi = {10.1093/sleep/zsaf061},
	pages = {zsaf061},
	journaltitle = {Sleep},
	shortjournal = {Sleep},
	author = {Fox, Benjamin and Jiang, Joy and Wickramaratne, Sajila and Kovatch, Patricia and Suarez-Farinas, Mayte and Shah, Neomi A. and Parekh, Ankit and Nadkarni, Girish N.},
	date = {2025-03-13},
	pmid = {40080690},
}

@article{deng_unified_2025,
	title = {A Unified Flexible Large Polysomnography Model for Sleep Staging and Mental Disorder Diagnosis},
	url = {https://www.ncbi.nlm.nih.gov/pmc/articles/PMC11661386/},
	doi = {10.1101/2024.12.11.24318815},
	pages = {2024.12.11.24318815},
	journaltitle = {{medRxiv}},
	shortjournal = {{medRxiv}},
	author = {Deng, Guifeng and Niu, Mengfan and Luo, Yuxi and Rao, Shuying and Xie, Junyi and Yu, Zhenghe and Liu, Wenjuan and Zhao, Sha and Pan, Gang and Li, Xiaojing and Deng, Wei and Guo, Wanjun and Li, Tao and Jiang, Haiteng},
	urldate = {2025-05-07},
	date = {2025-03-28},
	pmid = {39711704},
	pmcid = {PMC11661386},
}

@article{demvsar2006statistical,
  author  = {Janez Dem{\v{s}}ar},
  title   = {Statistical Comparisons of Classifiers over Multiple Data Sets},
  journal = {Journal of Machine Learning Research},
  year    = {2006},
  volume  = {7},
  number  = {1},
  pages   = {1--30},
  url     = {http://jmlr.org/papers/v7/demsar06a.html}
}

@inproceedings{yang2023bios,
 author = {Yang, Chaoqi and Westover, M and Sun, Jimeng},
 booktitle = {Advances in Neural Information Processing Systems},
 editor = {A. Oh and T. Naumann and A. Globerson and K. Saenko and M. Hardt and S. Levine},
 pages = {78240--78260},
 publisher = {Curran Associates, Inc.},
 title = {BIOT: Biosignal Transformer for Cross-data Learning in the Wild},
 url = {https://proceedings.neurips.cc/paper_files/paper/2023/file/f6b30f3e2dd9cb53bbf2024402d02295-Paper-Conference.pdf},
 volume = {36},
 year = {2023}
}

@ARTICLE{phan2022sleeptransformer,
  author={Phan, Huy and Mikkelsen, Kaare and Chén, Oliver Y. and Koch, Philipp and Mertins, Alfred and De Vos, Maarten},
  journal={IEEE Transactions on Biomedical Engineering}, 
  title={SleepTransformer: Automatic Sleep Staging With Interpretability and Uncertainty Quantification}, 
  year={2022},
  volume={69},
  number={8},
  pages={2456-2467},
  doi={10.1109/TBME.2022.3147187}}

@article{Guo2024transformer,
author = {Guo, Yanchen and Nowakowski, Maciej and Dai, Weiying},
year = {2024},
month = {11},
pages = {},
title = {FlexSleepTransformer: a transformer-based sleep staging model with flexible input channel configurations},
volume = {14},
journal = {Scientific Reports},
doi = {10.1038/s41598-024-76197-0}
}

@misc{li2016revisitingbatchnormalizationpractical,
      title={Revisiting Batch Normalization For Practical Domain Adaptation}, 
      author={Yanghao Li and Naiyan Wang and Jianping Shi and Jiaying Liu and Xiaodi Hou},
      year={2016},
      eprint={1603.04779},
      archivePrefix={arXiv},
      primaryClass={cs.CV},
      url={https://arxiv.org/abs/1603.04779}, 
}

@misc{chang2019domainspecificbatchnormalizationunsupervised,
      title={Domain-Specific Batch Normalization for Unsupervised Domain Adaptation}, 
      author={Woong-Gi Chang and Tackgeun You and Seonguk Seo and Suha Kwak and Bohyung Han},
      year={2019},
      eprint={1906.03950},
      archivePrefix={arXiv},
      primaryClass={cs.LG},
      url={https://arxiv.org/abs/1906.03950}, 
}
\bibliographystyle{iclr2026_conference}

\newpage
\appendix
\section{Appendix}
\subsection{Proof of the Bures-Wasserstein geodesic~\eqref{eq:running_barycenter_method} between covariance matrices of structure~\eqref{eq:covariance_structure}}
\label{app:geodesic}

\begin{proposition}
    Let $\bSigma^{(s)}$ and $\bSigma^{(t)}$ be two covariance matrices in $\bbR^{c\filtersize \times c\filtersize}$ following~\eqref{eq:covariance_structure}.
    Let us denote $\bP^{(s)}$ and $\bP^{(t)}$ the corresponding PSD matrices.
    The geodesic associated with the Bures-Wasserstein metric between $\bSigma^{(s)}$ and $\bSigma^{(t)}$ and parameterized by $\alpha \in [0,1]$ is $\bSigma(\alpha)$ following~\eqref{eq:covariance_structure} of PSD
    \begin{equation*}
        \bP(\alpha) = \left( \left(1 - \alpha\right) \bP^{(s)\odot\tfrac12} + \alpha \bP^{(t)\odot\tfrac12} \right)^{\odot 2} \;.
    \end{equation*}
\end{proposition}
\begin{proof}
    From~\cite{bhatia2019bures}, the geodesic associated with the Bures-Wasserstein metric between two covariance matrices $\bSigma^{(s)}$ and $\bSigma^{(t)}$ is given by
    \begin{equation}
        \gamma(\alpha) = (1 - \alpha)^2 \bSigma^{(s)} + \alpha^2 \bSigma^{(t)} + \alpha(1 - \alpha) \left[ (\bSigma^{(s)}\bSigma^{(t)})^{\frac12} + (\bSigma^{(t)}\bSigma^{(s)})^{\frac12} \right].
    \end{equation}
    where
    \begin{equation}
        (\bSigma^{(s)}\bSigma^{(t)})^{\frac12} = {\bSigma^{(s)}}^{\frac12} \left( {\bSigma^{(s)}}^{\frac12} \bSigma^{(t)} {\bSigma^{(s)}}^{\frac12} \right)^{\frac12} {\bSigma^{(s)}}^{-\frac12}.
    \end{equation}
    Since $\bSigma^{(s)}$ and $\bSigma^{(t)}$ diagonalize in the unitary basis $\bI_c \otimes \bF_{\filtersize}$, $\gamma(\alpha)$ also diagonalizes in this basis.
    Thus, we only have to compute the geodesic between the PSD matrices $\bP^{(s)}$ and $\bP^{(t)}$ and from now on, all operations are element-wise.
    Let $\bP(\alpha)$ be the PSD of $\gamma(\alpha)$, we have
    \begin{align}
        \bP(\alpha) &= (1-\alpha)^2 \bP^{(s)} + \alpha^2 \bP^{(t)} + \alpha(1-\alpha) \left[ (\bP^{(s)}\odot\bP^{(t)})^{\odot\frac12} + (\bP^{(t)}\odot\bP^{(s)})^{\odot\frac12} \right]\\
        &= (1-\alpha)^2 \bP^{(s)} + \alpha^2 \bP^{(t)} + 2\alpha(1-\alpha) (\bP^{(s)}\odot\bP^{(t)})^{\odot\frac12} \\
        &= \left( (1-\alpha) \bP^{(s)\odot\tfrac12} + \alpha \bP^{(t)\odot\tfrac12} \right)^{\odot 2}.
    \end{align}
    This concludes the proof.
\end{proof}

\subsection{Balanced datasets}
\begin{table}[ht]
    \centering
    \caption{Number of samples in the balanced datasets. Average and standard deviation (across LODO) are computed over 10 datasets left-out from the training set.}
    \label{tab:balanced_datasets}
    \begin{tabular}{lc}
    \toprule
    Balanced datasets & Number of subjects \\
    \midrule
    Balanced@40 & 360 $\pm$ 0 \\
    Balanced@100 & 787 $\pm$ 19 \\
    Balanced@200 & 1387 $\pm$ 63 \\
    Balanced@400 & 2466 $\pm$ 157 \\
    All subjects & 9929 $\pm$ 1659 \\
    \bottomrule
\end{tabular}
\end{table}
In the main paper, we report results across different training set sizes. 
Since the datasets are highly imbalanced (\eg ABC has 44 subjects, SHHS has 5,730),
we create balanced subsets by randomly selecting up to $N$ subjects per dataset. 
This avoids over-representing the largest dataset and ensures greater diversity in the training data. We consider four values of 
$N$: 40, 100, 200, and 400. The average number of subjects in each balanced set is shown in \cref{tab:balanced_datasets}. 
Notably, the balanced set with 400 subjects contains roughly four times less data than the full dataset.

\subsection{U-Time: CNN for time series segmentation}
\label{app:U-Sleep}

U-Time~\cite{perslev2019u,Perslev2021USleep} is a convolutional neural network (CNN) inspired by the U-Net architecture~\cite{ronneberger2015u}, designed for segmenting temporal sequences.
U-Time maps sequential inputs of arbitrary length to sequences of class labels on a freely chosen temporal scale.
The architecture is composed of several encoder and decoder blocks, with skip connections between them.

\paragraph{Encoder blocks}
A single encoder block is composed of a convolutional layer, an activation function, a BatchNorm layer, and a max pooling layer.
First, the convolution is applied to the input signal, followed by the activation function and the BatchNorm layer.
Finally, the max pooling layer downsamples the temporal dimension.
In the following, the pre-BatchNorm feature map is denoted $\mathbf{G}$ and the post-BatchNorm feature map $\widetilde{\mathbf{G}}$, \ie
$\widetilde{\mathbf{G}} \triangleq \text{BatchNorm}\left(\mathbf{G}\right)$.
Each encoder block downsamples by 2 the signal length but increases the number of channels.

\paragraph{Decoder blocks and Segmentation Head}
The decoding part of U-Time is symmetrical to the encoding part.
Each decoder block doubles the signal length and decreases the number of channels.
It is composed of a convolutional layer, an activation function, a BatchNorm layer, an upsampling layer and a concatenation layer of the skip connection of the corresponding encoding block.
Finally, the segmentation head applies two convolutional layers with an
activation function in between to output the final segmentation.
It should be noted that U-Time employs BatchNorm layers but other normalization layers, such as LayerNorm~\cite{ba2016layer} or InstanceNorm~\cite{ulyanov2016instance} are possible.

\paragraph{Implementation}
The architecture is inspired from Braindecode~\cite{schirrmeister2017deep}.
The implementation is improved to make it more efficient and faster. One epoch of training takes about 30 min on a single H100 GPU.

\subsection{Architecture: CNNTransformer}

The CNNTransformer is a hybrid architecture designed for multichannel time series classification inspired by transformers for EEG-Data~\cite{wang_caresleepnet_2024,phan2022sleeptransformer,yang2023bios,thapa_multimodal_2025}.
 It combines convolutional feature extraction with long-range temporal modeling via a Transformer encoder {at epoch-level}.
 The model processes an input tensor of shape $(B, S, C, T)$, where $B$ is the batch size, $S$ is the number of temporal segments, 
 $C$ is the number of input channels, and $T$ is the number of time samples per segment.
 {It outputs a tensor of shape $(B, n_{\text{classes}}, S)$, where $n_{\text{classes}}$ is the number of classes and $S$ is the number of epochs.}

The architecture consists of the following components:

\begin{itemize}
    \item \textbf{Reshaping:} The input is first permuted and reshaped to a 3D tensor of shape 
    $(B, C, S \cdot T)$ to be compatible with 1D convolutional layers applied along the temporal dimension.
    \item \textbf{CNN-based Feature Extractor:} A stack of 10 Conv1D layers, each followed by ELU activation 
    and Batch Normalization. Some layers use a stride greater than 1 to progressively reduce the temporal resolution. 
    This block extracts local temporal patterns and increases the representational capacity up to a dimensionality of $d_{\text{model}}$.
    \item \textbf{Adaptive Pooling:} An AdaptiveAvgPool1D layer reduces the temporal length to a fixed number of steps ($S$), 
    independent of the input sequence length. This step ensures a consistent temporal resolution before the Transformer.
    \item \textbf{Positional Encoding:} Learnable positional embeddings of shape $(1, S, d_{\text{model}})$ are added to the feature representations 
    to preserve temporal ordering before passing through the Transformer encoder.
    \item \textbf{Transformer Encoder:} A standard Transformer encoder composed of $L$ layers, each consisting of multi-head self-attention and a feedforward sublayer.
     This module models global temporal dependencies across the $S$ steps.
    \item \textbf{Classification Head:} After transposing the data to shape $(B, d_{\text{model}}, S)$, 
    a final 1D convolution with a kernel size of 1 projects the output to $n_{\text{classes}}$, yielding predictions for each epoch segment.
\end{itemize}

{The model is trained end-to-end using standard optimization techniques.}
The use of adaptive pooling and self-attention enables it to generalize across variable-length inputs while maintaining temporal resolution. 
A full summary of the architecture is provided in Table~\ref{tab:archi_cnntransformer}.

\label{sec:cnn_transformer}
\begin{table}[ht]
    \centering
    \caption{Architecture overview of the CNNTransformer model. {In pratice, $d_{\text{model}}$ is set to 768, $n_{\text{head}}$ to 8, and $S$ is 35.}}
    \resizebox{\textwidth}{!}{
    \begin{tabular}{llp{7cm}l}
    \toprule
    \textbf{Stage} & \textbf{Operation} & \textbf{Details} & \textbf{Output Shape} \\
    \midrule
    Input & Raw signal & Multichannel EEG signal with $S$ segments and $T$ time samples per segment & $(B, S, C, T)$ \\
    \midrule
    Reshape & Permute \& flatten & Rearranged as $(B, C, S \cdot T)$ to process with 1D convolutions & $(B, C, S \cdot T)$ \\
    \midrule
    Feature Extractor & 1D CNN stack & 10-layer sequence of Conv1D $\rightarrow$ ELU $\rightarrow$ BatchNorm; includes temporal downsampling via stride & $(B, d_{\text{model}}, T')$ \\
    \midrule
    Temporal Pooling & AdaptiveAvgPool1D & Downsamples to fixed temporal resolution defined by $S$ & $(B, d_{\text{model}}, S)$ \\
    \midrule
    Positional Encoding & Learnable embeddings & Added to temporal dimension to encode temporal order before transformer layers & $(B, d_{\text{model}}, S)$ \\
    \midrule
    Transformer Encoder & Multi-head attention & $2$ Transformer layers with $d_{\text{model}}$ embedding dimension, $n_{\text{head}}$ heads, and feedforward sublayers & $(B, d_{\text{model}}, S)$ \\
    \midrule
    Classifier & Linear projection & Projects feature vectors to class logits at each epoch time step & $(B, n_{\text{classes}}, S)$ \\
    \bottomrule
    \end{tabular}
    }
    \label{tab:archi_cnntransformer}
    \end{table}

\subsection{Equation for BatchNorm and InstanceNorm}
\paragraph{BatchNorm}
The BatchNorm layer~\cite{ioffe2015batch} normalizes features maps in a neural network to have zero mean and unit variance.
At train time, given a batch $\cB = \{\bG^{(1)}, \dots, \bG^{(N)}\} \subset \bbR^{c \times \ell}$ of $N$ pre-BatchNorm feature maps and for all $j,m,l \in \intset{N}\times \intset{c}\times \intset{\ell}$, the BatchNorm layer is computed as
\begin{equation}
    \label{eq:batch_norm}
    \widetilde{G}_{m,l}^{(j)} = \gamma_m \frac{G_{m,l}^{(j)} - \widehat{\mu}_m}{\sqrt{\widehat{\sigma}_m^2 + \varepsilon}} + \beta_m \;,
\end{equation}
where $\bgamma, \bbeta \in \bbR^{c}$ are learnable parameters.
The mean and standard deviation $\widehat{\bmu} \in \bbR^{c}$ and $\widehat{\bsigma} \in \bbR^{c}$ are computed across the time and the batch,
\begin{equation}
    \begin{aligned}
        \widehat{\mu}_m &\triangleq \frac{1}{N\ell}\sum_{j=1}^{N}\sum_{l=1}^{\ell} G_{m,l}^{(j)},\\
        \widehat{\sigma}_m^2 &\triangleq \frac{1}{N\ell}\sum_{j=1}^{N}\sum_{l=1}^{\ell} \left(G_{m,l}^{(j)} - \widehat{\mu}_m\right)^{2}.
    \end{aligned}
\end{equation}
At test time, the mean and variance $\widehat{\bmu}$ and $\widehat{\bsigma}$ are replaced by their running mean and variance, also called exponential moving average, estimated during training.

\paragraph{InstanceNorm}
Another popular normalization is the InstanceNorm layer~\cite{ulyanov2016instance}.
During training, InstanceNorm operates similarly to~\eqref{eq:batch_norm}, but the mean and variance are computed per sample instead of across the batch dimension, \ie $\widehat{\mu}_m^{(j)}$ and $\widehat{\sigma}_m^{(j)}$ are computed for each sample $j$,
\begin{equation}
    \begin{aligned}
        \widehat{\mu}_m^{(j)} &\triangleq \frac{1}{\ell}\sum_{l=1}^{\ell} G_{m,l}^{(j)} \;, \\
        (\widehat{\sigma}_m^{(j)})^2 &\triangleq \frac{1}{\ell} \sum_{l=1}^{\ell} \left(G_{m,l}^{(j)} - \widehat{\mu}_m^{(j)}\right)^{2} \; .
    \end{aligned}
\end{equation}
Hence, each sensor of each sample is normalized independently of the others.
At test time, InstanceNorm behaves identically to its training phase and therefore does not rely on running statistics contrary to the BatchNorm.

\subsection{Sensitivity to Filter Size}
\label{app:filter_size}
\begin{figure}
    \centering
    \includegraphics[width=0.7\linewidth]{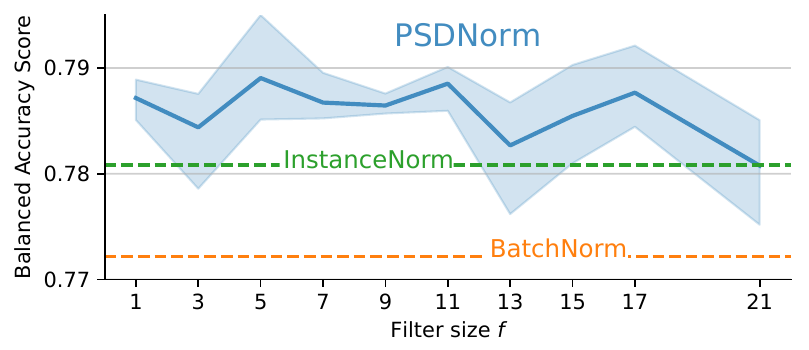}
    \captionof{figure}{
        \textbf{Performance of \method with varying filter sizes.}
        The BACC score is plotted against the filter size used with U-Sleep.
    }
    \label{fig:sensitive}
\end{figure}
{The filter size $\filtersize$ in \method controls the temporal context used for normalization, influencing the strength of adaptation to temporal variations.
\Cref{fig:sensitive} shows the impact of different $\filtersize$ values on the BACC score across datasets using U-Sleep trained on balanced@400.
This experiments shows that for any $\filtersize$, \method consistently improves performance compared to other normalization techniques.
Taking a $\filtersize$ between 5 and 11 yields the best results, with a peak at $\filtersize=5$.
Smaller values (e.g., $\filtersize=1$, equivalent to InstanceNorm) provide less adaptation, while larger values (e.g., $\filtersize=21$) may over-smooth temporal variations, leading to diminished performance.
Overall, the results shows that $\filtersize$ is not so sensitive yielding good performance for a wide range of values.}

\subsection{F1 score vs. Balanced Accuracy}
In the main paper, we report Balanced Accuracy scores, which account for class imbalance in sleep stage classification. 
Prior work, such as the U-Time paper~\cite{perslev2019u}, uses the F1 score to evaluate performance. 
In Table~\ref{tab:f1_scores}, we report F1 scores on the left-out datasets. 
These scores are slightly higher than the Balanced Accuracy scores and are comparable to those reported in the U-Time paper.

Our main findings remain consistent: BatchNorm and InstanceNorm are the strongest baselines and achieve the best performance on 3 out of 10 datasets. 
PSDNorm outperforms all other methods on 7 out of 10 datasets. 
The same trend holds for the balanced@400 setup, where PSDNorm again outperforms all baselines on 7 datasets, 
while InstanceNorm is never the top performer.

These results confirm that our implementation achieves state-of-the-art performance in sleep stage classification.
Moreover, PSDNorm maintains its advantage even in data-limited settings
\begin{table*}[ht]
    \centering
    \caption{
        \textbf{F1 scores of different methods on the left-out datasets.}
        The lower section displays results for training over datasets balanced @400 \ie \textbf{small-scale dataset}, while the upper section presents results for training over all subjects \ie \textbf{large-scale dataset}.
        The best scores are highlighted in \textbf{bold}.
        The reported standard deviations indicate performance variability across 3 seeds.
    }
    \resizebox{\textwidth}{!}{
        \begin{tabular}{lcccccc}
        \toprule
        & Dataset & BatchNorm & LayerNorm & InstanceNorm & TMA & PSDNorm(F=5) \\
        \midrule
        \multirow[t]{11}{*}{\rotatebox[origin=r]{90}{All subjects\;\;\;\;\;\;\;\;\;\;}} & ABC & $81.00_{\pm 0.11}$ & $79.50_{\pm 0.49}$ & $80.56_{\pm 0.39}$ & $80.89_{\pm 0.06}$ & $\mathbf{81.12_{\pm 0.37}}$ \\
        & CCSHS & $\mathbf{89.83_{\pm 0.19}}$ & $89.01_{\pm 0.43}$ & $89.39_{\pm 0.16}$ & $89.37_{\pm 0.11}$ & $89.13_{\pm 0.17}$ \\
        & CFS & $88.30_{\pm 0.52}$ & $87.39_{\pm 0.06}$ & $88.45_{\pm 0.17}$ & $88.28_{\pm 0.37}$ & $\mathbf{88.52_{\pm 0.15}}$ \\
        & CHAT & $65.77_{\pm 4.06}$ & $65.25_{\pm 3.96}$ & $71.35_{\pm 2.75}$ & $71.80_{\pm 2.66}$ & $\mathbf{72.16_{\pm 2.21}}$ \\
        & HOMEPAP & $77.06_{\pm 0.14}$ & $76.62_{\pm 1.06}$ & $77.50_{\pm 0.46}$ & $\mathbf{77.82_{\pm 0.64}}$ & $77.30_{\pm 0.24}$ \\
        & MASS & $77.27_{\pm 1.42}$ & $74.21_{\pm 2.05}$ & $75.12_{\pm 2.08}$ & $\mathbf{77.74_{\pm 1.05}}$ & $76.00_{\pm 3.00}$ \\
        & MROS & $\mathbf{85.53_{\pm 0.48}}$ & $84.02_{\pm 0.95}$ & $85.22_{\pm 0.19}$ & $85.13_{\pm 0.98}$ & $85.02_{\pm 0.42}$ \\
        & PhysioNet & $74.98_{\pm 1.84}$ & $74.29_{\pm 1.50}$ & $75.07_{\pm 1.05}$ & $\mathbf{76.01_{\pm 0.73}}$ & $75.29_{\pm 1.21}$ \\
        & SHHS & $78.95_{\pm 0.92}$ & $78.04_{\pm 1.21}$ & $80.30_{\pm 1.29}$ & $78.84_{\pm 0.43}$ & $\mathbf{80.32_{\pm 0.91}}$ \\
        & SOF & $86.30_{\pm 0.40}$ & $85.82_{\pm 0.22}$ & $86.57_{\pm 0.60}$ & $86.31_{\pm 0.27}$ & $\mathbf{86.99_{\pm 0.33}}$ \\
        \midrule
        & Mean(Dataset) & $80.50_{\pm 0.51}$ & $79.41_{\pm 0.73}$ & $80.95_{\pm 0.36}$ & $\mathbf{81.22_{\pm 0.20}}$ & $81.19_{\pm 0.11}$ \\
        & Mean(Subject) & $80.05_{\pm 0.78}$ & $79.09_{\pm 0.90}$ & $81.31_{\pm 0.83}$ & $80.59_{\pm 0.19}$ & $\mathbf{81.39_{\pm 0.69}}$ \\
        \midrule
        \multirow[t]{11}{*}{\rotatebox[origin=r]{90}{Balanced@400 \;\;\;\;\;\;\;\;\;}} & ABC & $\mathbf{79.80_{\pm 0.34}}$ & $77.86_{\pm 0.80}$ & $78.36_{\pm 1.20}$ & $79.49_{\pm 0.68}$ & $78.08_{\pm 0.78}$ \\
        & CCSHS & $88.32_{\pm 0.49}$ & $87.22_{\pm 0.51}$ & $88.73_{\pm 0.52}$ & $88.47_{\pm 0.62}$ & $\mathbf{88.79_{\pm 0.99}}$ \\
        & CFS & $87.01_{\pm 0.18}$ & $85.61_{\pm 0.16}$ & $\mathbf{87.62_{\pm 0.27}}$ & $87.37_{\pm 0.44}$ & $87.06_{\pm 0.77}$ \\
        & CHAT & $66.56_{\pm 1.42}$ & $61.32_{\pm 2.25}$ & $64.19_{\pm 4.63}$ & $69.90_{\pm 2.74}$ & $\mathbf{71.86_{\pm 0.95}}$ \\
        & HOMEPAP & $76.20_{\pm 1.25}$ & $76.15_{\pm 1.13}$ & $77.66_{\pm 0.58}$ & $76.83_{\pm 0.97}$ & $\mathbf{77.85_{\pm 1.29}}$ \\
        & MASS & $76.06_{\pm 1.69}$ & $73.95_{\pm 5.80}$ & $76.94_{\pm 1.12}$ & $76.32_{\pm 0.36}$ & $\mathbf{77.16_{\pm 1.73}}$ \\
        & MROS & $83.69_{\pm 0.39}$ & $82.22_{\pm 1.27}$ & $83.95_{\pm 0.53}$ & $\mathbf{84.15_{\pm 0.46}}$ & $83.51_{\pm 0.84}$ \\
        & PhysioNet & $\mathbf{76.26_{\pm 1.27}}$ & $70.40_{\pm 0.14}$ & $73.84_{\pm 0.93}$ & $75.24_{\pm 2.72}$ & $73.51_{\pm 3.05}$ \\
        & SHHS & $76.98_{\pm 0.70}$ & $75.98_{\pm 0.22}$ & $79.12_{\pm 0.96}$ & $78.19_{\pm 0.90}$ & $\mathbf{79.26_{\pm 1.35}}$ \\
        & SOF & $85.49_{\pm 0.58}$ & $84.23_{\pm 1.30}$ & $85.50_{\pm 0.86}$ & $\mathbf{85.56_{\pm 0.90}}$ & $84.14_{\pm 1.05}$ \\
        \midrule
        & Mean(Dataset) & $79.64_{\pm 0.41}$ & $77.57_{\pm 0.73}$ & $79.59_{\pm 0.25}$ & $\mathbf{80.15_{\pm 0.26}}$ & $80.12_{\pm 0.57}$ \\
         & Mean(Subject)  & $78.57_{\pm 0.55}$ & $76.86_{\pm 0.22}$ & $79.53_{\pm 0.30}$ & $79.70_{\pm 0.66}$ & $\mathbf{80.29_{\pm 0.68}}$ \\
        \bottomrule
        \end{tabular}
    }
    
    \label{tab:f1_scores}
\end{table*}
\subsection{Impact of Whitening and Target Covariance}
\label{app:whitening}
As explained in the main paper, InstanceNorm is a special case of \method\ 
with $F = 1$ and an identity target covariance matrix (i.e., whitening). \method\ 
extends this by (i) using temporal context with $F > 1$, and 
(ii) mapping the PSD to a target covariance matrix, such as a barycenter 
(i.e., colorization).

In this section, we evaluate the impact of whitening on the performance of \method, 
to assess the benefit of using a barycenter as the target covariance matrix.
Table~\ref{tab:whitening} reports results on 10 datasets (balanced@400), 
with and without whitening.

Whitening improves performance on only one dataset (CCSHS), while projecting 
to the barycenter yields the best results on 6 datasets.

This suggests that, while whitening may help when $F = 1$, it is less effective 
when $F > 1$. Using a barycenter leads to a more robust and stable target 
covariance matrix.

\begin{table*}[ht]
    \centering
    \caption{Impact of the whitening on the performance of \method on the 10 datasets balanced @ 400.}
        \begin{tabular}{lllll}
                    \toprule
                     \multirow{2}{*}{Dataset}& \multirow{2}{*}{BatchNorm} & \multirow{2}{*}{InstanceNorm} & \multicolumn{2}{c}{PSDNorm} \\
                     \cmidrule{4-5}
                     &  & & Barycenter & Whitening \\
                    \midrule
                    ABC & $78.26_{\pm 1.33}$ & $\mathbf{78.73_{\pm 0.42}}$ & $78.18_{\pm 0.68}$ & $77.86_{\pm 1.33}$ \\
                    CCSHS & $87.42_{\pm 0.16}$ & $87.62_{\pm 0.42}$ & $87.58_{\pm 0.30}$ & $\mathbf{87.80_{\pm 0.23}}$ \\
                    CFS & $84.32_{\pm 0.57}$ & $\mathbf{84.72_{\pm 0.33}}$ & $84.29_{\pm 0.36}$ & $84.01_{\pm 0.60}$ \\
                    CHAT & $66.55_{\pm 0.88}$ & $64.43_{\pm 4.41}$ & $\mathbf{70.28_{\pm 1.70}}$ & $69.07_{\pm 3.73}$ \\
                    HOMEPAP & $75.25_{\pm 0.50}$ & $76.47_{\pm 0.63}$ &  $\mathbf{76.83_{\pm 0.61}}$ & $76.13_{\pm 0.93}$ \\
                    MASS & $70.00_{\pm 1.91}$ & $71.52_{\pm 1.13}$ & $\mathbf{72.77_{\pm 1.09}}$ & $69.11_{\pm 1.51}$ \\
                    MROS & $80.37_{\pm 0.20}$ & $80.28_{\pm 0.21}$ & $80.26_{\pm 0.11}$ & $\mathbf{80.50_{\pm 0.75}}$ \\
                    PhysioNet & $\mathbf{75.81_{\pm 0.13}}$ & $74.68_{\pm 0.55}$ & $74.82_{\pm 2.11}$ & $74.58_{\pm 1.57}$ \\
                    SHHS & $76.44_{\pm 0.92}$ & $78.68_{\pm 0.37}$ &  $\mathbf{78.88_{\pm 0.68}}$ & $78.77_{\pm 0.67}$ \\
                    SOF & $81.08_{\pm 1.14}$ & $80.68_{\pm 1.38}$ & $79.49_{\pm 0.41}$ & $80.10_{\pm 0.62}$ \\
                    \hline
                    Mean & $77.55_{\pm 0.34}$ & $77.78_{\pm 0.46}$ & $\mathbf{78.34_{\pm 0.42}}$ & $77.79_{\pm 0.30}$ \\
                    \bottomrule
                    \end{tabular}
    \label{tab:whitening}
\end{table*}

\subsection{Generalization of PSDNorm in CNNTransformer}
\label{app:cnntransformer}

The CNNTransformer architecture is a hybrid model that combines 
convolutional and transformer layers for time series classification.

The main paper presents a critical difference diagram for the CNNTransformer 
evaluated on datasets balanced@400. It shows that PSDNorm with $F=5$ 
is the best-performing normalization layer.

In Table~\ref{tab:cnntransformer}, we report the results of different normalization 
layers used in the CNNTransformer architecture on datasets balanced@400.

First, we observe that CNNTransformer performs slightly below U-Sleep.
Second, BatchNorm and InstanceNorm are the best performers 
on one and two datasets respectively, while PSDNorm achieves the 
best performance on 7 out of 10 datasets.

PSDNorm with $F=5$ outperforms BatchNorm by a margin of 0.9 
and InstanceNorm by 0.54 in average score.

These results highlight that PSDNorm is a plug-and-play normalization 
layer that can be seamlessly integrated into various architectures to reduce 
feature space variability.

\begin{table*}
    \centering
    \caption{Different normalization layers used in the CNNTransformer architecture for datasets balanced@400.}
        \begin{tabular}{lcccc}
        \toprule
        Dataset & BatchNorm &  InstanceNorm & TMA & PSDNorm \\
        \midrule
        ABC & $76.99_{\pm 0.53}$ & $75.40_{\pm 0.36}$ & $\mathbf{77.50_{\pm 0.54}}$ & $76.31_{\pm 0.46}$ \\
        CCSHS & $86.75_{\pm 0.48}$ & $\mathbf{87.00_{\pm 0.34}}$ & $86.73_{\pm 0.25}$ & $86.92_{\pm 0.32}$ \\
        CFS & $83.32_{\pm 0.35}$ & $\mathbf{83.77_{\pm 0.34}}$ & $83.16_{\pm 0.38}$ & $83.71_{\pm 0.29}$ \\
        CHAT & $66.44_{\pm 0.49}$ & $66.40_{\pm 2.55}$ & $66.47_{\pm 1.37}$ & $\mathbf{70.04_{\pm 0.37}}$ \\
        HOMEPAP & $74.81_{\pm 1.36}$ & $\mathbf{75.92_{\pm 0.44}}$ & $74.76_{\pm 0.83}$ & $75.26_{\pm 0.55}$ \\
        MASS & $71.51_{\pm 0.47}$ & $71.70_{\pm 1.17}$ & $70.57_{\pm 0.80}$ & $\mathbf{72.55_{\pm 0.81}}$ \\
        MROS & $79.77_{\pm 0.31}$ & $79.74_{\pm 0.55}$ & $\mathbf{79.85_{\pm 0.08}}$ & $79.77_{\pm 0.30}$ \\
        PhysioNet & $72.54_{\pm 0.34}$ & $74.36_{\pm 0.84}$ & $71.39_{\pm 1.38}$ & $\mathbf{74.95_{\pm 0.41}}$ \\
        SHHS & $75.34_{\pm 0.34}$ & $76.55_{\pm 0.92}$ & $75.15_{\pm 0.98}$ & $\mathbf{77.26_{\pm 0.57}}$ \\
        SOF & $80.63_{\pm 0.60}$  & $80.78_{\pm 0.54}$ & $\mathbf{81.03_{\pm 0.48}}$ & $80.31_{\pm 0.90}$ \\
        \midrule
        Mean & $76.38_{\pm 0.17}$ & $77.07_{\pm 0.28}$ & $76.30_{\pm 0.57}$ & $\mathbf{77.83_{\pm 0.36}}$ \\
        \bottomrule
        \end{tabular}
    \label{tab:cnntransformer}
\end{table*}

\subsection{Evolution of Performance with Training Set Size}
\begin{figure}
    \centering
    \includegraphics[width=0.7\linewidth]{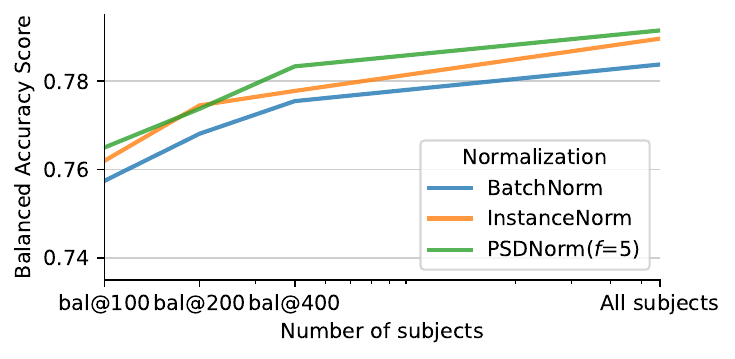}
    \captionof{figure}{
        \textbf{Performance of \method and BatchNorm with varying training set sizes.}
        The BACC score is plotted against the number of training subjects used with U-Sleep.
    }
    \label{fig:number_subjects}
\end{figure}
The choice of $\filtersize$ in \method controls the intensity of the normalization: larger $\filtersize$ 
provide stronger normalization, while smaller $\filtersize$ allow more flexibility in the model.
In \Cref{fig:number_subjects}, we evaluate its impact across different training set sizes and 
observe a clear trend: when trained on fewer subjects, larger filter sizes yield better performance (\ie $\filtersize=17$), 
whereas smaller filter sizes are more effective with larger datasets (\ie $\filtersize={5}$).
This suggests that with limited data, stronger normalization helps prevent overfitting, 
while with more data, a more flexible model is preferred.
On average, \method\ with $\filtersize=5$ offers a good compromise, 
achieving one of the best performances across all training set sizes.

\subsection{Critical Difference Diagram for U-Sleep on all subjects}
\label{app:critical_diagram}
The main paper presents the critical difference diagram for U-Sleep on the dataset balanced@400.  
Figure~\ref{fig:critical_diagram_all_subjects} extends this analysis to all subjects across datasets.
The conclusion remains consistent: PSDNorm with $F=5$ is the best-performing normalization layer,  
while BatchNorm performs the worst.
Interestingly, PSDNorm with $F=17$ ranks second to last, suggesting that overly strong adaptation  
can hurt performance when the dataset is large.

\begin{figure}[ht]
    \centering
    \includegraphics[width=0.8\textwidth]{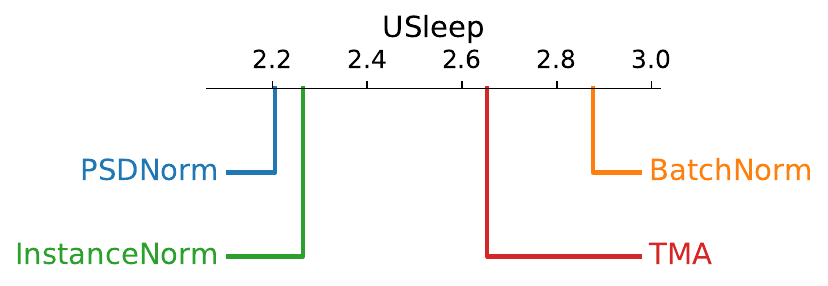}
    \caption{Critical difference diagram for U-Sleep on all subjects.}
    \label{fig:critical_diagram_all_subjects}
\end{figure}

\subsection{Computational time of PSDNorm}
\label{app:psdnorm_time}
\begin{table}[!ht]
    \caption{\textbf{Computational time of PSDNorm compared to BatchNorm and InstanceNorm for USleep and CNNTransformer.}
    The time is done for leave out the dataset CHAT and with the dataset balanced @400.
    The time is averaged over 3 runs and reported in seconds.}
    \centering
    \label{tab:psdnorm_time}
    \resizebox{\textwidth}{!}{

    \begin{tabular}{llll}
    \toprule
        Model & Normalization & Time per epoch (sec) & Time of inference (sec) \\
        \midrule
        USleep & BatchNorm & $\mathbf{161.94 \pm 10.18}$ & $95.63 \pm 4.85$ \\ 
        USleep & InstanceNorm & $258.71 \pm 2.15 (*)$ & $\mathbf{93.40 \pm 8.18}$ \\ 
        USleep & PSDNorm($\filtersize=5$) & $172.85 \pm 4.05$ & $98.72 \pm 12.09$ \\
        \midrule
        CNNTransformer & BatchNorm & $130.88 \pm 2.67$ & $93.03 \pm 5.25$ \\ 
        CNNTransformer & InstanceNorm & $\mathbf{127.47 \pm 5.14}$ & $92.70 \pm 4.11$ \\ 
        CNNTransformer & PSDNorm($\filtersize=5$) & $152.83 \pm 2.29$ & $\mathbf{92.57 \pm 2.64}$ \\ 
        \bottomrule
    \end{tabular}
    }
\end{table}

One important aspect of normalization layers is their computational cost, which can impact training and inference times.
Table~\ref{tab:psdnorm_time} compares the computational time of PSDNorm with BatchNorm and InstanceNorm in both U-Sleep and CNNTransformer architectures.
In U-Sleep, PSDNorm takes 172.85 seconds per epoch, which is slightly higher than BatchNorm (161.94 seconds) but significantly lower than InstanceNorm (258.71 seconds).
The high cost of InstanceNorm is due to the fact that the torch.compile was not working for Usleep and InstanceNorm.
For inference, PSDNorm takes 98.72 seconds, which is comparable to BatchNorm (95.63 seconds) but slightly higher than InstanceNorm (93.40 seconds).

In CNNTransformer, PSDNorm takes 152.83 seconds per epoch, which is higher than BatchNorm (130.88 seconds) and InstanceNorm (127.47 seconds).
However, for inference, PSDNorm is equivalent to both BatchNorm and InstanceNorm.
The modest computational overhead introduced by PSDNorm is a worthwhile trade-off for its superior performance. 
This efficiency is enabled by the highly optimized implementation of the Fast Fourier Transform (FFT) on GPUs.

\subsection{Class-wise performance}
\label{app:classwise_performance}

\begin{table}
    \caption{Class-wise F1 scores for BatchNorm layer on datasets balanced @ 400.}
    \centering
    \begin{tabular}{lllllll}
    \toprule
    Dataset & Wake & N1 & N2 & N3 & REM & F1 \\
    \midrule
    ABC & $86.07_{\pm 0.68}$ & $53.97_{\pm 0.36}$ & $80.05_{\pm 0.52}$ & $71.35_{\pm 1.45}$ & $88.43_{\pm 0.13}$ & $79.80_{\pm 0.34}$ \\
    CCSHS & $95.22_{\pm 0.46}$ & $48.08_{\pm 2.87}$ & $84.88_{\pm 0.81}$ & $86.13_{\pm 1.01}$ & $88.54_{\pm 0.93}$ & $88.32_{\pm 0.49}$ \\
    CFS & $94.64_{\pm 0.10}$ & $42.69_{\pm 0.39}$ & $82.27_{\pm 0.75}$ & $77.20_{\pm 0.30}$ & $86.94_{\pm 0.34}$ & $87.01_{\pm 0.18}$ \\
    CHAT & $78.35_{\pm 1.31}$ & $35.57_{\pm 2.64}$ & $52.63_{\pm 2.77}$ & $72.67_{\pm 1.25}$ & $76.12_{\pm 1.50}$ & $66.56_{\pm 1.42}$ \\
    HOMEPAP & $84.51_{\pm 1.01}$ & $41.45_{\pm 0.71}$ & $73.88_{\pm 2.00}$ & $57.41_{\pm 1.72}$ & $82.52_{\pm 1.33}$ & $76.20_{\pm 1.25}$ \\
    MASS & $66.93_{\pm 4.65}$ & $40.02_{\pm 2.00}$ & $78.73_{\pm 0.67}$ & $67.01_{\pm 0.39}$ & $75.80_{\pm 5.65}$ & $76.06_{\pm 1.69}$ \\
    MROS & $94.63_{\pm 0.25}$ & $41.73_{\pm 1.11}$ & $73.74_{\pm 0.72}$ & $47.86_{\pm 0.33}$ & $82.70_{\pm 0.46}$ & $83.69_{\pm 0.39}$ \\
    PhysioNet & $89.22_{\pm 0.49}$ & $46.01_{\pm 0.90}$ & $73.93_{\pm 3.04}$ & $55.04_{\pm 1.56}$ & $77.30_{\pm 0.61}$ & $76.26_{\pm 1.27}$ \\
    SHHS & $85.68_{\pm 1.57}$ & $32.56_{\pm 1.34}$ & $72.48_{\pm 2.00}$ & $61.91_{\pm 1.31}$ & $79.09_{\pm 1.07}$ & $76.98_{\pm 0.70}$ \\
    SOF & $93.91_{\pm 0.15}$ & $38.29_{\pm 1.09}$ & $79.15_{\pm 0.57}$ & $71.86_{\pm 3.17}$ & $86.49_{\pm 0.18}$ & $85.49_{\pm 0.58}$ \\
    \midrule
    Mean & $86.91_{\pm 1.07}$ & $42.04_{\pm 1.34}$ & $75.17_{\pm 1.39}$ & $66.84_{\pm 1.25}$ & $82.39_{\pm 1.22}$ & $79.64_{\pm 0.83}$ \\
    \bottomrule
    \end{tabular}
\end{table}

\begin{table}
    \caption{Class-wise F1 scores for LayerNorm layer with $\filtersize=5$ on datasets balanced @ 400.}
    \centering
    \begin{tabular}{lllllll}
    \toprule
    Dataset & Wake & N1 & N2 & N3 & REM & F1 \\
    \midrule
    ABC & $83.29_{\pm 2.94}$ & $52.93_{\pm 0.96}$ & $78.07_{\pm 1.04}$ & $68.91_{\pm 2.93}$ & $85.05_{\pm 0.81}$ & $77.86_{\pm 0.80}$ \\
    CCSHS & $93.66_{\pm 0.40}$ & $40.68_{\pm 1.11}$ & $84.80_{\pm 1.28}$ & $86.46_{\pm 0.82}$ & $85.39_{\pm 0.62}$ & $87.22_{\pm 0.51}$ \\
    CFS & $93.78_{\pm 0.50}$ & $38.56_{\pm 2.07}$ & $81.15_{\pm 0.11}$ & $75.39_{\pm 1.77}$ & $83.68_{\pm 0.90}$ & $85.61_{\pm 0.16}$ \\
    CHAT & $72.07_{\pm 2.84}$ & $30.11_{\pm 0.63}$ & $46.82_{\pm 6.98}$ & $70.45_{\pm 2.92}$ & $68.11_{\pm 4.07}$ & $61.32_{\pm 2.25}$ \\
    HOMEPAP & $83.58_{\pm 1.72}$ & $43.85_{\pm 1.55}$ & $73.97_{\pm 1.76}$ & $57.54_{\pm 1.92}$ & $79.16_{\pm 1.19}$ & $76.15_{\pm 1.13}$ \\
    MASS & $60.73_{\pm 4.02}$ & $40.56_{\pm 0.87}$ & $77.39_{\pm 6.28}$ & $65.91_{\pm 4.01}$ & $69.92_{\pm 10.53}$ & $73.95_{\pm 5.80}$ \\
    MROS & $94.12_{\pm 0.10}$ & $38.41_{\pm 2.21}$ & $71.48_{\pm 3.66}$ & $47.85_{\pm 0.37}$ & $79.05_{\pm 0.66}$ & $82.22_{\pm 1.27}$ \\
    PhysioNet & $88.04_{\pm 0.37}$ & $44.72_{\pm 2.03}$ & $64.53_{\pm 1.93}$ & $48.49_{\pm 0.54}$ & $68.17_{\pm 7.15}$ & $70.40_{\pm 0.14}$ \\
    SHHS & $84.61_{\pm 2.49}$ & $32.89_{\pm 1.92}$ & $72.02_{\pm 2.48}$ & $61.13_{\pm 1.59}$ & $75.63_{\pm 0.63}$ & $75.98_{\pm 0.22}$ \\
    SOF & $93.24_{\pm 0.27}$ & $35.61_{\pm 2.49}$ & $77.49_{\pm 2.15}$ & $70.83_{\pm 3.09}$ & $83.52_{\pm 0.01}$ & $84.23_{\pm 1.30}$ \\
    \midrule
    Mean & $84.71_{\pm 1.56}$ & $39.83_{\pm 1.58}$ & $72.77_{\pm 2.77}$ & $65.30_{\pm 1.99}$ & $77.77_{\pm 2.66}$ & $77.49_{\pm 1.36}$ \\
    \bottomrule
    \end{tabular}
\end{table}

\begin{table}
    \caption{Class-wise F1 scores for Instancenorm layer on datasets balanced @ 400.}
    \centering
    \begin{tabular}{lllllll}
\toprule
Dataset & Wake & N1 & N2 & N3 & REM & F1 \\
\midrule
ABC & $86.46_{\pm 1.75}$ & $54.45_{\pm 0.78}$ & $75.68_{\pm 2.16}$ & $70.75_{\pm 0.51}$ & $88.60_{\pm 0.75}$ & $78.36_{\pm 1.20}$ \\
CCSHS & $95.69_{\pm 0.14}$ & $49.23_{\pm 2.62}$ & $85.29_{\pm 0.89}$ & $85.98_{\pm 0.68}$ & $89.23_{\pm 0.91}$ & $88.73_{\pm 0.52}$ \\
CFS & $94.95_{\pm 0.19}$ & $44.97_{\pm 1.34}$ & $83.30_{\pm 0.60}$ & $77.17_{\pm 0.39}$ & $87.45_{\pm 0.38}$ & $87.62_{\pm 0.27}$ \\
CHAT & $77.67_{\pm 6.38}$ & $29.47_{\pm 8.45}$ & $48.78_{\pm 6.21}$ & $71.34_{\pm 2.64}$ & $74.58_{\pm 1.83}$ & $64.19_{\pm 4.63}$ \\
HOMEPAP & $86.27_{\pm 0.53}$ & $43.19_{\pm 1.46}$ & $75.19_{\pm 1.07}$ & $58.39_{\pm 1.03}$ & $83.44_{\pm 0.46}$ & $77.66_{\pm 0.58}$ \\
MASS & $67.25_{\pm 1.95}$ & $43.19_{\pm 2.05}$ & $78.64_{\pm 1.76}$ & $65.78_{\pm 1.37}$ & $78.34_{\pm 1.18}$ & $76.94_{\pm 1.12}$ \\
MROS & $94.80_{\pm 0.18}$ & $41.46_{\pm 0.71}$ & $74.42_{\pm 1.16}$ & $48.89_{\pm 2.31}$ & $82.11_{\pm 0.10}$ & $83.95_{\pm 0.53}$ \\
PhysioNet & $89.43_{\pm 0.41}$ & $44.35_{\pm 0.62}$ & $68.91_{\pm 2.82}$ & $51.05_{\pm 1.32}$ & $77.35_{\pm 0.95}$ & $73.84_{\pm 0.93}$ \\
SHHS & $88.62_{\pm 0.30}$ & $33.02_{\pm 2.26}$ & $74.31_{\pm 2.07}$ & $64.28_{\pm 0.72}$ & $80.32_{\pm 0.42}$ & $79.12_{\pm 0.96}$ \\
SOF & $94.42_{\pm 0.20}$ & $37.18_{\pm 2.37}$ & $78.43_{\pm 1.88}$ & $72.39_{\pm 1.56}$ & $86.82_{\pm 0.70}$ & $85.50_{\pm 0.86}$ \\
    \midrule
Mean & $87.55_{\pm 1.20}$ & $42.05_{\pm 2.27}$ & $74.29_{\pm 2.06}$ & $66.60_{\pm 1.25}$ & $82.83_{\pm 0.77}$ & $79.59_{\pm 1.16}$ \\
\bottomrule
\end{tabular}
\end{table}

\begin{table}
    \caption{Class-wise F1 scores for TMA preprocessing with $\filtersize=5$ on datasets balanced @ 400.}
    \centering
\begin{tabular}{lllllll}
\toprule
Dataset & Wake & N1 & N2 & N3 & REM & F1 \\
\midrule
ABC & $85.50_{\pm 0.93}$ & $54.76_{\pm 0.88}$ & $78.91_{\pm 1.27}$ & $71.40_{\pm 1.23}$ & $88.43_{\pm 0.47}$ & $79.49_{\pm 0.68}$ \\
CCSHS & $95.51_{\pm 0.31}$ & $48.72_{\pm 2.07}$ & $85.02_{\pm 1.09}$ & $85.37_{\pm 1.33}$ & $89.33_{\pm 0.29}$ & $88.47_{\pm 0.62}$ \\
CFS & $94.75_{\pm 0.26}$ & $43.28_{\pm 2.27}$ & $83.06_{\pm 0.79}$ & $77.38_{\pm 0.28}$ & $87.17_{\pm 0.36}$ & $87.37_{\pm 0.44}$ \\
CHAT & $80.70_{\pm 4.88}$ & $37.51_{\pm 1.85}$ & $58.89_{\pm 4.08}$ & $75.71_{\pm 2.09}$ & $75.95_{\pm 2.35}$ & $69.90_{\pm 2.74}$ \\
HOMEPAP & $84.13_{\pm 0.98}$ & $43.92_{\pm 0.51}$ & $74.83_{\pm 1.40}$ & $57.94_{\pm 1.19}$ & $81.58_{\pm 0.13}$ & $76.83_{\pm 0.97}$ \\
MASS & $70.45_{\pm 7.07}$ & $41.83_{\pm 3.63}$ & $77.30_{\pm 1.48}$ & $65.04_{\pm 2.01}$ & $79.50_{\pm 2.42}$ & $76.32_{\pm 0.36}$ \\
MROS & $94.50_{\pm 0.32}$ & $41.74_{\pm 1.13}$ & $75.20_{\pm 1.26}$ & $48.92_{\pm 2.07}$ & $82.35_{\pm 0.70}$ & $84.15_{\pm 0.46}$ \\
PhysioNet & $89.40_{\pm 0.50}$ & $44.65_{\pm 2.19}$ & $71.13_{\pm 5.36}$ & $51.61_{\pm 3.94}$ & $80.06_{\pm 1.02}$ & $75.24_{\pm 2.72}$ \\
SHHS & $88.30_{\pm 0.58}$ & $32.95_{\pm 2.17}$ & $73.23_{\pm 2.58}$ & $62.33_{\pm 0.33}$ & $79.27_{\pm 0.84}$ & $78.19_{\pm 0.90}$ \\
SOF & $93.69_{\pm 0.34}$ & $37.85_{\pm 2.16}$ & $79.33_{\pm 1.49}$ & $72.46_{\pm 1.92}$ & $86.83_{\pm 0.29}$ & $85.56_{\pm 0.90}$ \\
\midrule
Mean & $87.69_{\pm 1.62}$ & $42.72_{\pm 1.89}$ & $75.69_{\pm 2.08}$ & $66.82_{\pm 1.64}$ & $83.05_{\pm 0.89}$ & $80.15_{\pm 1.08}$ \\
\bottomrule
\end{tabular}
\end{table}

\begin{table}
    \caption{Class-wise F1 scores for PSDNorm layer with $\filtersize=5$ on datasets balanced @ 400.}
    \centering
    \begin{tabular}{lllllll}
    \toprule
    Dataset & Wake & N1 & N2 & N3 & REM & F1 \\
    \midrule
    ABC & $84.57_{\pm 1.39}$ & $54.46_{\pm 0.59}$ & $75.94_{\pm 1.14}$ & $70.73_{\pm 1.00}$ & $88.19_{\pm 0.23}$ & $78.08_{\pm 0.78}$ \\
    CCSHS & $95.76_{\pm 0.21}$ & $47.97_{\pm 1.71}$ & $85.34_{\pm 1.76}$ & $86.17_{\pm 2.24}$ & $89.71_{\pm 0.33}$ & $88.79_{\pm 0.99}$ \\
    CFS & $95.01_{\pm 0.17}$ & $42.92_{\pm 1.06}$ & $82.08_{\pm 1.88}$ & $76.98_{\pm 0.96}$ & $87.31_{\pm 0.13}$ & $87.06_{\pm 0.77}$ \\
    CHAT & $82.93_{\pm 1.99}$ & $36.59_{\pm 6.79}$ & $61.84_{\pm 2.98}$ & $77.06_{\pm 1.32}$ & $78.01_{\pm 1.62}$ & $71.86_{\pm 0.95}$ \\
    HOMEPAP & $85.75_{\pm 1.45}$ & $44.47_{\pm 0.78}$ & $75.54_{\pm 1.72}$ & $58.70_{\pm 1.40}$ & $83.24_{\pm 0.55}$ & $77.85_{\pm 1.29}$ \\
    MASS & $72.74_{\pm 2.12}$ & $42.20_{\pm 1.16}$ & $78.56_{\pm 2.98}$ & $66.13_{\pm 2.37}$ & $78.23_{\pm 3.21}$ & $77.16_{\pm 1.73}$ \\
    MROS & $94.63_{\pm 0.31}$ & $41.40_{\pm 1.64}$ & $73.33_{\pm 1.80}$ & $47.56_{\pm 2.25}$ & $82.52_{\pm 0.47}$ & $83.51_{\pm 0.84}$ \\
    PhysioNet & $89.48_{\pm 0.44}$ & $44.33_{\pm 1.43}$ & $67.67_{\pm 6.16}$ & $49.37_{\pm 4.91}$ & $79.23_{\pm 1.21}$ & $73.51_{\pm 3.05}$ \\
    SHHS & $89.09_{\pm 0.66}$ & $33.78_{\pm 2.53}$ & $74.15_{\pm 2.70}$ & $64.40_{\pm 0.29}$ & $80.24_{\pm 1.03}$ & $79.26_{\pm 1.35}$ \\
    SOF & $93.74_{\pm 0.27}$ & $34.70_{\pm 2.45}$ & $76.58_{\pm 2.97}$ & $70.20_{\pm 2.02}$ & $85.50_{\pm 1.63}$ & $84.14_{\pm 1.05}$ \\
    \midrule
    Mean & $88.37_{\pm 0.90}$ & $42.28_{\pm 2.01}$ & $75.10_{\pm 2.61}$ & $66.73_{\pm 1.88}$ & $83.22_{\pm 1.04}$ & $80.12_{\pm 1.28}$ \\
    \bottomrule
    \end{tabular}
\end{table}

\begin{figure}
    \centering
    \includegraphics[width=0.7\linewidth]{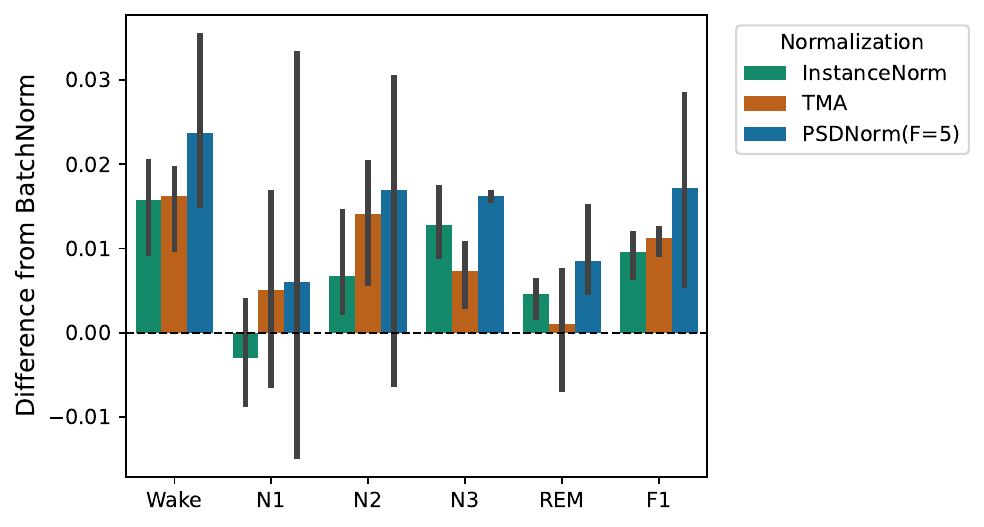}
    \captionof{figure}{Class-wise F1 score differences between normalization layers. The variance is giving by the seeds.}
    \label{fig:classwise_f1_diff}
\end{figure}

The Tables show that the most complicated sleep stage to classify is N1, with F1 scores consistently lower than other stages across all normalization methods.
This is likely due to the inherent difficulty of distinguishing N1 from other stages, as it shares characteristics with both wakefulness and deeper sleep stages.
In contrast, stages like Wake and REM tend to have higher F1 scores, indicating that they are easier to classify accurately.

Figure~\ref{fig:classwise_f1_diff} illustrates the class-wise F1 score differences between normalization layers score against BatchNorm score.
For almost all the classes, other normalization increase the performance compared to BatchNorm except for N1 where InstanceNorm shows a decrease in performance.
PSDNorm is consistently the best performing normalization across all classes, highlighting its effectiveness in improving sleep stage classification.
But we have to note that for N1 and N2 the variance over the seed is big showing the instability of the training for these classes.

\subsection{Study of impact of number of layer in U-Sleep using PSDNorm}
\label{app:usleep_layers}
\begin{figure}
    \centering
    \includegraphics[width=0.7\linewidth]{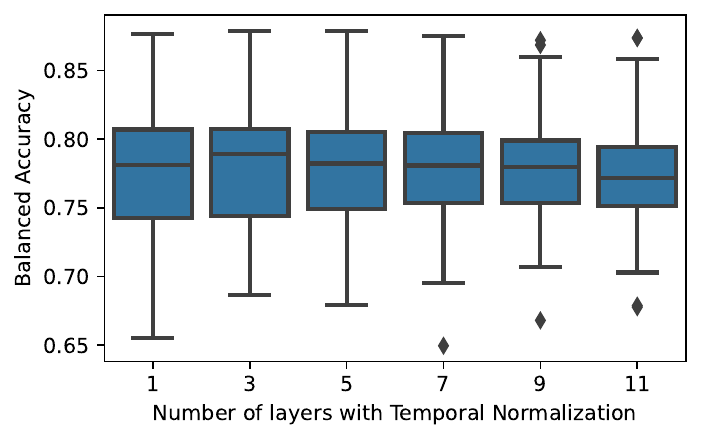}
    \captionof{figure}{Impact of the number of layers in U-Sleep using PSDNorm with $\filtersize=5$. The BACC score is plotted against the number of layers. The variance is one the Datasets.}
    \label{fig:usleep_layers}
\end{figure}

In the main paper, we apply PSDNorm in 3 layers of U-Sleep. Here, we investigate the impact of varying the number of layers that utilize PSDNorm.
Figure~\ref{fig:usleep_layers} shows the BACC score as a function of the number of layers with PSDNorm.
The results indicate that increasing the number of layers with PSDNorm reach a plateau after 3 layers, but does reduce the variance across datasets.
This suggests that while adding more layers with PSDNorm can enhance performance, there are diminishing returns beyond a certain point.
Thus, using PSDNorm in 3 layers strikes a good balance between performance, good variance, and computational efficiency.

\subsection{Study in very low data regime}
\label{app:low_data_regime}
\begin{table}
    \centering
    \caption{Performance of different normalization layers in U-Sleep in very low data regime on datasets balanced @ 40.}
    \begin{tabular}{lccccc}
    \toprule
     Dataset & BatchNorm & LayerNorm & InstanceNorm & PSDNorm(F=5) & PSDNorm(F=15) \\
    \midrule
    ABC & $\mathbf{73.83_{\pm 1.80}}$ & $64.54_{\pm 3.46}$ & $72.29_{\pm 1.50}$ & $71.35_{\pm 1.15}$ & $72.61_{\pm 1.84}$ \\
    CCSHS & $83.58_{\pm 0.45}$ & $77.91_{\pm 1.59}$ & $\mathbf{85.33_{\pm 0.69}}$ & $85.10_{\pm 0.17}$ & $85.00_{\pm 0.25}$ \\
    CFS & $81.13_{\pm 0.85}$ & $76.57_{\pm 1.89}$ & $81.59_{\pm 0.47}$ & $\mathbf{81.79_{\pm 0.82}}$ & $80.93_{\pm 0.13}$ \\
    CHAT & $55.74_{\pm 2.38}$ & $58.42_{\pm 1.64}$ & $63.38_{\pm 5.26}$ & $59.66_{\pm 0.89}$ & $\mathbf{67.86_{\pm 3.59}}$ \\
    HOMEPAP & $74.52_{\pm 1.66}$ & $72.19_{\pm 1.65}$ & $76.03_{\pm 0.48}$ & $76.01_{\pm 0.32}$ & $\mathbf{76.14_{\pm 1.63}}$ \\
    MASS & $\mathbf{70.15_{\pm 3.09}}$ & $64.79_{\pm 2.72}$ & $66.58_{\pm 0.30}$ & $69.49_{\pm 1.12}$ & $68.21_{\pm 6.25}$ \\
    MROS & $77.12_{\pm 0.03}$ & $71.76_{\pm 2.38}$ & $76.59_{\pm 0.28}$ & $\mathbf{77.19_{\pm 0.38}}$ & $76.77_{\pm 1.30}$ \\
    PhysioNet & $71.68_{\pm 1.61}$ & $69.59_{\pm 1.46}$ & $72.68_{\pm 3.09}$ & $\mathbf{73.67_{\pm 0.93}}$ & $72.08_{\pm 3.65}$ \\
    SHHS & $73.74_{\pm 1.11}$ & $71.50_{\pm 0.97}$ & $75.56_{\pm 1.66}$ & $75.43_{\pm 1.38}$ & $\mathbf{76.00_{\pm 0.63}}$ \\
    SOF & $75.84_{\pm 2.16}$ & $73.50_{\pm 1.97}$ & $\mathbf{76.54_{\pm 1.59}}$ & $75.14_{\pm 1.79}$ & $76.00_{\pm 3.00}$ \\
    \midrule
    Mean & $73.35_{\pm 0.87}$ & $70.72_{\pm 1.22}$ & $75.19_{\pm 1.03}$ & $74.79_{\pm 0.75}$ & $\mathbf{75.88_{\pm 0.93}}$ \\
    \bottomrule
    \end{tabular}
\end{table}

In this section, we explore the performance of different normalization layers in U-Sleep when trained on a very limited dataset, specifically balanced @ 40 subjects.
The results indicate that in this low data regime, PSDNorm with a small filter size ($F=5$) struggle to outperform InstanceNorm while still outperforming BatchNorm and LayerNorm.
However, PSDNorm with a larger filter size ($F=15$) gives the best average performance across datasets with an increase of more than 10\% in BACC compared to other normalization layers for CHAT dataset.
This suggests that in scenarios with very limited data, stronger normalization (larger filter size) is beneficial to prevent overfitting and enhance generalization.

\subsection{Comparison with AdaBN}
\begin{table}
    \centering
\caption{\textbf{Comparison of \method with AdaBN on datasets balanced @400 using U-Sleep.} }
\resizebox{\textwidth}{!}{
\begin{tabular}{lcccccccc}
\toprule
Dataset & BatchNorm & LayerNorm & InstanceNorm & AdaBN(3) & AdaBN(12) & AdaBN(full) & TMA & PSDNorm \\
\midrule
ABC & $78.26_{\pm 1.33}$ & $75.29_{\pm 0.81}$ & $\mathbf{78.73_{\pm 0.42}}$ & $78.25_{\pm 1.30}$ & $77.21_{\pm 1.56}$ & $76.89_{\pm 1.30}$ & $78.04_{\pm 0.51}$ & $78.18_{\pm 0.68}$ \\
CCSHS & $87.42_{\pm 0.16}$ & $85.20_{\pm 0.48}$ & $\mathbf{87.62_{\pm 0.42}}$ & $87.38_{\pm 0.17}$ & NaN & $86.99_{\pm nan}$ & $87.57_{\pm 0.20}$ & $87.58_{\pm 0.30}$ \\
CFS & $84.32_{\pm 0.57}$ & $81.66_{\pm 1.36}$ & $\mathbf{84.72_{\pm 0.33}}$ & $84.21_{\pm 0.60}$ & NaN & $83.61_{\pm nan}$ & $84.58_{\pm 0.20}$ & $84.29_{\pm 0.36}$ \\
CHAT & $66.55_{\pm 0.88}$ & $61.19_{\pm 1.16}$ & $64.43_{\pm 4.41}$ & $66.49_{\pm 0.89}$ & NaN & NaN & $68.73_{\pm 2.48}$ & $\mathbf{70.28_{\pm 1.70}}$ \\
HOMEPAP & $75.25_{\pm 0.50}$ & $74.86_{\pm 0.25}$ & $76.47_{\pm 0.63}$ & $75.15_{\pm 0.46}$ & $74.39_{\pm 0.56}$ & $74.46_{\pm 0.53}$ & $76.10_{\pm 0.32}$ & $\mathbf{76.83_{\pm 0.61}}$ \\
MASS & $70.00_{\pm 1.91}$ & $68.56_{\pm 3.33}$ & $71.52_{\pm 1.13}$ & $69.68_{\pm 1.66}$ & $68.46_{\pm 2.58}$ & $68.31_{\pm 1.86}$ & $71.63_{\pm 1.92}$ & $\mathbf{72.77_{\pm 1.09}}$ \\
MROS & $\mathbf{80.37_{\pm 0.20}}$ & $78.05_{\pm 0.22}$ & $80.28_{\pm 0.21}$ & $80.34_{\pm 0.20}$ & NaN & NaN & $80.09_{\pm 0.40}$ & $80.26_{\pm 0.11}$ \\
PhysioNet & $\mathbf{75.81_{\pm 0.13}}$ & $71.82_{\pm 2.12}$ & $74.68_{\pm 0.55}$ & $75.27_{\pm 0.14}$ & $74.01_{\pm 0.13}$ & $74.01_{\pm 0.14}$ & $75.31_{\pm 1.54}$ & $74.82_{\pm 2.11}$ \\
SHHS & $76.44_{\pm 0.92}$ & $75.12_{\pm 0.39}$ & $78.68_{\pm 0.37}$ & $76.43_{\pm 0.92}$ & NaN & NaN & $77.00_{\pm 0.39}$ & $\mathbf{78.88_{\pm 0.68}}$ \\
SOF & $81.08_{\pm 1.14}$ & $78.70_{\pm 0.50}$ & $80.68_{\pm 1.38}$ & $81.05_{\pm 1.13}$ & NaN & $81.17_{\pm nan}$ & $\mathbf{81.25_{\pm 0.71}}$ & $79.49_{\pm 0.41}$ \\
\midrule
Mean & $77.22_{\pm 0.34}$ & $75.04_{\pm 0.42}$ & $78.17_{\pm 0.28}$ & $77.18_{\pm 0.34}$ & $74.29_{\pm 1.08}$ & $76.59_{\pm 4.82}$ & $77.74_{\pm 0.36}$ & $\mathbf{78.85_{\pm 0.59}}$ \\

\bottomrule
\end{tabular}
}
\label{tab:adabn}
\end{table}

Table \ref{tab:adabn} presents a comparison between \method and AdaBN using the U-Sleep architecture on datasets balanced @400.
AdaBN adapts the BatchNorm statistics separately for each subject.
 In the original paper, all BN layers are replaced (AdaBN(full)), 
 but for a fair comparison we also evaluate two additional settings: AdaBN(3), which adapts only the first three BN layers, 
 and AdaBN(12), which adapts only the first BN layers of the encoders.

As expected, AdaBN struggles to achieve strong performance on sleep staging. 
It consistently underperforms compared to TMA and, in some cases, even performs worse than standard BatchNorm. 
Notably, increasing the number of adapted BN layers further degrades performance, highlighting the importance of not adapting too many layers within the model.
In contrast, \method consistently outperforms AdaBN across all datasets, demonstrating its effectiveness in normalizing features for sleep staging tasks.


\end{document}